\documentclass[letterpaper]{article}
\pdfoutput=1
\usepackage{amsmath,amscd,amssymb,amsthm}
\usepackage{textcomp, libertine}
\usepackage{thmtools}
\usepackage{thm-restate}

\declaretheorem{theorem}
\declaretheorem[sibling=theorem]{lemma}
\declaretheorem[sibling=theorem]{corollary}
\declaretheorem[sibling=theorem]{definition}

\newcommand{\R}{\mathbb{R}}
\newcommand{\E}{\mathop{\mathbb{E}}}
\newcommand{\argmin}{\mathop{\text{argmin}}}
\newcommand{\sign}{\text{sign}}
\newcommand{\Lm}{L_{\max}}

\usepackage{algorithm}
\usepackage{algorithmic}

\newcommand{\algname}{\textsc{rescaledexp}}

\newcommand{\scaleinvariant}{\textsc{ScaleInvariant}}

\newcommand{\adadelta}{\textsc{AdaDelta}}
\newcommand{\adagrad}{\textsc{AdaGrad}}
\newcommand{\pistol}{\textsc{PiSTOL}}
\newcommand{\adam}{\textsc{Adam}}

\usepackage{graphicx}
\usepackage{subfigure} 

\PassOptionsToPackage{square,sort,comma,numbers}{natbib}
\usepackage[final]{nips_2016}

\usepackage[utf8]{inputenc} 
\usepackage[T1]{fontenc}    
\usepackage{hyperref}       
\usepackage{url}            
\usepackage{booktabs}       
\usepackage{amsfonts}       
\usepackage{nicefrac}       
\usepackage{microtype}      

\title{Online Convex Optimization with Unconstrained Domains and Losses}

\author{
  Ashok Cutkosky\\
  Department of Computer Science\\
  Stanford University\\
  \texttt{ashokc@cs.stanford.edu} \\
  \And
  Kwabena Boahen \\
  Department of Bioengineering \\
  Stanford University \\
  \texttt{boahen@stanford.edu} \\
}

\begin{document}

\maketitle

\begin{abstract}
  We propose an online convex optimization algorithm (\algname) that achieves optimal regret in the unconstrained setting without prior knowledge of any bounds on the loss functions.
  We prove a lower bound showing an \emph{exponential} separation between the regret of existing algorithms that require a known bound on the loss functions and any algorithm that does not require such knowledge. \algname\ matches this lower bound asymptotically in the number of iterations. \algname\ is naturally hyperparameter-free and we demonstrate empirically that it matches prior optimization algorithms that require hyperparameter optimization.
\end{abstract}

\section{Online Convex Optimization}
Online Convex Optimization (OCO) \citep{zinkevich2003online,shalev2011online} provides an elegant framework for modeling noisy, antagonistic or changing environments. The problem can be stated formally with the help of the following definitions:
\begin{description}
    \item[Convex Set:] A set $W$ is convex if $W$ is contained in some real vector space and $tw+(1-t)w'\in W$ for all $w,w'\in W$ and $t\in[0,1]$.
    \item[Convex Function:] $f:W\to\R$ is a convex function if $f(tw+(1-t)w')\le tf(w)+(1-t)f(w')$ for all $w,w'\in W$ and $t\in[0,1]$.
\end{description}
An OCO problem is a game of repeated rounds in which on round $t$ a learner first chooses an element $w_t$ in some convex space $W$, then receives a convex loss function $\ell_t$, and suffers loss $\ell_t(w_t)$. The \emph{regret} of the learner with respect to some other $u\in W$ is defined by
\[
R_T(u)=\sum_{t=1}^T \ell_t(w_t)-\ell_t(u)
\]
The objective is to design an algorithm that can achieve low regret with respect to any $u$, even in the face of adversarially chosen $\ell_t$. 

Many practical problems can be formulated as OCO problems. For example, the stochastic optimization problems found widely throughout machine learning have exactly the same form, but with i.i.d. loss functions, a subset of the OCO problems. In this setting the goal is to identify a vector $w_\star$ with low generalization error  ($\E[\ell(w_\star)-\ell(u)]$). We can solve this by running an OCO algorithm for $T$ rounds and setting $w_\star$ to be the average value of $w_t$. By online-to-batch conversion results \citep{littlestone2014line,cesa2004generalization}, the generalization error is bounded by the expectation of the regret over the $\ell_t$ divided by $T$. Thus, OCO algorithms can be used to solve stochastic optimization problems while also performing well in non-i.i.d. settings.

The regret of an OCO problem is upper-bounded by the regret on a corresponding Online Linear Optimization (OLO) problem, in which each $\ell_t$ is further constrained to be a \emph{linear} function: $\ell_t(w)=g_t\cdot w_t$ for some $g_t$. The reduction follows, with the help of one more definition: 
\begin{description}
    \item[Subgradient:] $g\in W$ is a subgradient of $f$ at $w$, denoted $g\in\partial f(w)$, if and only if $f(w)+g\cdot (w'-w)\le f(w')$ for all $w'$. Note that $\partial f(w)\ne \emptyset$ if $f$ is convex.\footnote{In full generality, a subgradient is an element of the dual space $W^*$. However, we will only consider cases where the subgradient is naturally identified with an element in the original space $W$ (e.g. $W$ is finite dimensional) so that the definition in terms of dot-products suffices.}
\end{description}
To reduce OCO to OLO, suppose $g_t\in \partial \ell_t(w_t)$, and consider replacing $\ell_t(w)$ with the linear approximation $g_t\cdot w$. Then using the definition of subgradient,
\[
R_T(u)=\sum_{t=1}^T \ell_t(w_t)-\ell_t(u)\le \sum_{t=1}^T g_t(w_t-u)=\sum_{t=1}^T g_tw_t-g_tu
\]
so that replacing $\ell_t(w)$ with $g_t\cdot w$ can only make the problem more difficult. All of the analysis in this paper therefore addresses OLO, accessing convex losses functions only through subgradients.

There are two major factors that influence the regret of OLO algorithms: the size of the space $W$ and the size of the subgradients $g_t$. When $W$ is a bounded set (the ``constrained'' case), then given $B=\max_{w\in W}\|w\|$, there exist OLO algorithms \citep{duchi10adagrad,mcmahan2010adaptive} that can achieve $R_T(u)\le O\left(B\Lm\sqrt{T}\right)$ without knowing $\Lm = \max_t \|g_t\|$. When $W$ is unbounded (the ``unconstrained'' case), then given $\Lm$, there exist algorithms \citep{mcmahan2012no,orabona2013dimension,mcmahan2013minimax} that achieve $
R_T(u)\le \tilde O(\|u\|\log(\|u\|)\Lm\sqrt{T})$ or $R_t(u)\le \tilde O(\|u\|\sqrt{\log(\|u\|)}\Lm \sqrt{T})$,
where $\tilde O$ hides factors that depend logarithmically on $\Lm$ and $T$. These algorithms are known to be optimal (up to constants) for their respective regimes \citep{abernethy2008optimal,mcmahan2012no}. All algorithms for the unconstrained setting to-date require knowledge of $\Lm$ to achieve these optimal bounds.\footnote{There are algorithms that do not require $\Lm$, but achieve only regret $O(\|u\|^2)$ \citep{orabona2016scale}} Thus a natural question is: \emph{can we achieve $O(\|u\|\log(\|u\|))$ regret in the unconstrained, unknown-$\Lm$ setting?} This problem has been posed as a COLT 2016 open problem \citep{orabona2016open}, and is solved in this paper.

A simple approach is to maintain an estimate of $\Lm$ and double it whenever we see a new $g_t$ that violates the assumed bound (the so-called ``doubling trick''), thereby turning a known-$\Lm$ algorithm into an unknown-$\Lm$ algorithm. This strategy fails for previous known-$\Lm$ algorithms because their analysis makes strong use of the assumption that each and every $\|g_t\|$ is bounded by $\Lm$. The existence of even a small number of bound-violating $g_t$ can throw off the entire analysis.

In this paper, we prove that it is actually impossible to achieve regret $O\left(\|u\|\log(\|u\|)\Lm\sqrt{T}+\Lm\exp\left[\left(\max_{t}\tfrac{\|g_t\|}{L(t)}\right)^{1/2-\epsilon}\right]\right)$ for any $\epsilon>0$ where $\Lm$ and $L(t)=\max_{t'<t} \|g_{t'}\|$ are unknown in advance (Section \ref{sec:lowerbound}). This immediately rules out the ``ideal'' bound of $\tilde O(\|u\|\sqrt{\log(\|u\|)}\Lm\sqrt{T})$ which is possible in the known-$\Lm$ case. Secondly, we provide an algorithm, \algname, that matches our lower bound without prior knowledge of $\Lm$, leading to a naturally hyperparameter-free algorithm (Section \ref{sec:algname}). To our knowledge, this is the first algorithm to address the unknown-$\Lm$ issue while maintaining $O(\|u\|\log \|u\|)$ dependence on $u$. Finally, we present empirical results showing that \algname\ performs well in practice (Section \ref{sec:experiments}).

\section{Lower Bound with Unknown $\Lm$}\label{sec:lowerbound}

The following theorem rules out algorithms that achieve regret $O(u\log(u)\Lm\sqrt{T})$ without prior knowledge of $\Lm$. In fact, any such algorithm must pay an up-front penalty that is \emph{exponential} in $T$. This lower bound resolves a COLT 2016 open problem (Parameter-Free and Scale-Free Online Algorithms) \citep{orabona2016open} in the negative.

\begin{theorem}\label{thm:lowerbound}
For any constants $c,k,\epsilon>0$, there exists a $T$ and an adversarial strategy picking $g_t\in \R$ in response to $w_{t}\in \R$ such that regret is:
\begin{align*}
R_T(u)&=\sum_{t=1}^T g_tw_t-g_tu\\
&\ge (k+c\|u\|\log \|u\|) \Lm\sqrt{T}\log(\Lm+1)+k\Lm\exp((2T)^{1/2-\epsilon})\\
&\ge (k+c\|u\|\log \|u\|) \Lm\sqrt{T}\log(\Lm+1) + k\Lm\exp\left[\left(\max_{t}\frac{\|g_t\|}{L(t)}\right)^{1/2-\epsilon}\right]
\end{align*}
for some $u\in \R$ where $\Lm=\max_{t\le T} \|g_t\|$ and $L(t)=\max_{t'<t} \|g_{t'}\|$.
\end{theorem}
\begin{proof}
We prove the theorem by showing that for sufficiently large $T$, the adversary can ``checkmate'' the learner by presenting it only with the subgradient $g_t=-1$. If the learner fails to have $w_t$ increase quickly, then there is a $u\gg1$ against which the learner has high regret. On the other hand, if the learner ever does make $w_t$ higher than a particular threshold, the adversary immediately punishes the learner with a subgradient $g_t=2T$, again resulting in high regret.

Let $T$ be large enough such that both of the following hold: 
\begin{align}
\tfrac{T}{4}\exp(\tfrac{T^{1/2}}{4\log(2)c})&>k\log(2)\sqrt{T}+k\exp((2T)^{1/2-\epsilon})\label{eqn:firstcondition}\\
\tfrac{T}{2}\exp(\tfrac{T^{1/2}}{4\log(2)c})&>2kT\exp((2T)^{1/2-\epsilon})+2kT\sqrt{T}\log(2T+1)\label{eqn:secondcondition}
\end{align}
The adversary plays the following strategy: for all $t\le T$, so long as $w_t<\frac{1}{2}\exp(T^{1/2}/4\log(2)c)$, give $g_t=-1$. As soon as $w_t\ge \frac{1}{2}\exp(T^{1/2}/4\log(2)c)$, give $g_t=2T$ and $g_t=0$ for all subsequent $t$. Let's analyze the regret at time $T$ in these two cases.

\noindent{\bf Case 1: $w_t<\frac{1}{2}\exp(T^{1/2}/4\log(2)c)$ for all $t$:}

In this case, let $u=\exp(T^{1/2}/4\log(2)c)$. Then $\Lm=1$, $\max_t \tfrac{\|g_t\|}{L(t)}=1$, and using (\ref{eqn:firstcondition}) the learner's regret is at least
\begin{align*}
R_T(u)&\ge Tu-T\frac{1}{2}\exp(\tfrac{T^{1/2}}{4\log(2)c})\\
&=\tfrac{1}{2}Tu\\
& = cu\log(u)\sqrt{T}\log(2)+\tfrac{T}{4}\exp(\tfrac{T^{1/2}}{4\log(2)c})\\
&>cu\log(u)\Lm\sqrt{T}\log(\Lm+1)+k\Lm\sqrt{T}\log(\Lm+1)+k\Lm\exp((2T)^{1/2-\epsilon})\\
&=(k+cu\log u) \Lm\sqrt{T}\log(\Lm+1) + k\Lm\exp\left[\left(2T\right)^{1/2-\epsilon}\right]
\end{align*}

\noindent{\bf Case 2: $w_t\ge\frac{1}{2} \exp(T^{1/2}/4\log(2)c)$ for some $t$:} 

In this case, $\Lm=2T$ and $\max_t \tfrac{\|g_t\|}{L(t)}=2T$. For $u=0$, using (\ref{eqn:secondcondition}), the regret is at least
\begin{align*}
R_T(u)&\ge \tfrac{T}{2}\exp(\tfrac{T^{1/2}}{4\log(2)c})\\
&\ge 2kT\exp((2T)^{1/2-\epsilon})+2kT\sqrt{T}\log(2T+1)\\
&= k\Lm\exp((2T)^{1/2-\epsilon})+k\Lm\sqrt{T}\log(\Lm+1)\\
&=(k+cu\log u) \Lm\sqrt{T}\log(\Lm+1) + k\Lm\exp\left[\left(2T\right)^{1/2-\epsilon}\right]
\end{align*}
\end{proof}

The exponential lower-bound arises because the learner has to move exponentially fast in order to deal with exponentially far away $u$, but then experiences exponential regret if the adversary provides a gradient of unprecedented magnitude in the opposite direction. However, if we play against an adversary that is constrained to give loss vectors $\|g_t\|\le\Lm$ for some $\Lm$ that does not grow with time, or if the losses do not grow too quickly, then we can still achieve $R_T(u)=O(\|u\|\log(\|u\|)\Lm\sqrt{T})$ asymptotically without knowing $\Lm$. In the following sections we describe an algorithm that accomplishes this.

\section{\algname}\label{sec:algname}

Our algorithm, \algname, adapts to the unknown $\Lm$ using a guess-and-double strategy that is robust to a small number of bound-violating $g_t$s. We initialize a guess $L$ for $\Lm$ to $\|g_1\|$. Then we run a novel known-$\Lm$ algorithm that can achieve good regret in the unconstrained $u$ setting. As soon as we see a $g_t$ with $\|g_t\|>2L$, we update our guess to $\|g_t\|$ and restart the known-$\Lm$ algorithm. To prove that this scheme is effective, we show (Lemma \ref{thm:firstregret}) that our known-$\Lm$ algorithm does not suffer too much regret when it sees a $g_t$ that violates its assumed bound.

Our known-$\Lm$ algorithm uses the Follow-the-Regularized-Leader (FTRL) framework. FTRL is an intuitive way to design OCO algorithms \citep{shalev07online}: Given functions $\psi_t:W\to \R$, at time $T$ we play $w_T=\argmin\left[\psi_{T-1}(w)+\sum_{t=1}^{T-1}\ell_t(w)\right]$. The functions $\psi_t$ are called regularizers. A large number of OCO algorithms (e.g. gradient descent) can be cleanly formulated as instances of this framework.

Our known-$\Lm$ algorithm is FTRL with regularizers $\psi_t(w)=\psi(w)/\eta_t$, where $\psi(w)=(\|w\|+1)\log(\|w\|+1)-\|w\|$ and $\eta_t$ is a scale-factor that we adapt over time. Specifically, we set $\eta_t^{-1} = k\sqrt{2}\sqrt{M_t+\|g\|^2_{1:t}}$, where we use the compressed sum notations $g_{1:T}=\sum_{t=1}^T g_t$ and $\|g\|^2_{1:T}=\sum_{t=1}^T\|g_t\|^2$. $M_t$ is defined recursively by $M_0=0$ and $M_{t}=\max(M_{t-1},\|g_{1:t}\|/p-\|g\|^2_{1:t})$, so that $M_{t}\ge M_{t-1}$, and $M_t+\|g\|^2_{1:t}\ge \|g_{1:t}\|/p$. $k$ and $p$ are constants: $k=\sqrt{2}$ and $p=\Lm^{-1}$.

\algname's strategy is to maintain an estimate $L_t$ of $\Lm$ at all time steps. Whenever it observes $\|g_t\|\ge 2L_{t}$, it updates $L_{t+1}=\|g_t\|$. We call periods during which $L_t$ is constant \emph{epochs}. Every time it updates $L_t$, it restarts our known-$\Lm$ algorithm with $p=\frac{1}{L_t}$, beginning a new epoch. Notice that since $L_t$ at least doubles every epoch, there will be at most $\log_2(\Lm/L_1)+1$ total epochs. To address edge cases, we set $w_t=0$ until we suffer a non-constant loss function, and we set the initial value of $L_t$ to be the first non-zero $g_t$. Pseudo-code is given in Algorithm \ref{alg:algname}, and Theorem \ref{thm:fullregret} states our regret bound. For simplicity, we re-index so that that $g_1$ is the first non-zero gradient received. No regret is suffered when $g_t=0$ so this does not affect our analysis.

\begin{algorithm}
   \caption{\algname}
   \label{alg:algname}
\begin{algorithmic}
   \STATE {\bfseries Initialize:} $k\gets\sqrt{2}$, $M_0\gets 0$, $w_1\gets 0$, $t_\star\gets 1$  // $t_\star$ is the start-time of the current epoch.
   \FOR{$t=1$ {\bfseries to} $T$}
   \STATE Play $w_t$, receive subgradient $g_t\in \partial \ell_t(w_t)$.
   \IF{$t=1$}
   \STATE $L_{1}\gets \|g_{1}\|$
    \STATE $p\gets1/L_1$  
   \ENDIF
   \STATE $M_t \gets \max(M_{t-1},\|g_{t_\star:t}\|/p-\|g\|^2_{t_\star:t})$.
   \STATE $\eta_t\gets \frac{1}{k\sqrt{2(M_t+\|g\|^2_{t_\star:t})}}$
   \STATE //Set $w_{t+1}$ using FTRL update
   \STATE $w_{t+1} \gets -\frac{g_{t_\star:t}}{\|g_{t_\star:t}\|}\left[\exp(\eta_t\|g_{t_\star:t}\|)-1\right]$ // $=\argmin_w\left[\frac{\psi(w)}{\eta_t}+g_{t_\star:t}w\right]$
   \IF{$\|g_{t}\|>2L_{t}$}
   \STATE //Begin a new epoch: update $L$ and restart FTRL
   \STATE $L_{t+1}\gets \|g_{t}\|$
   \STATE $p\gets1/L_{t+1}$  
   \STATE $t_\star\gets t+1$
   \STATE $M_t\gets 0$
   \STATE $w_{t+1}\gets0$

   \ELSE
   \STATE $L_{t+1}\gets L_{t}$
   \ENDIF
   
   \ENDFOR
\end{algorithmic}
\end{algorithm}

\begin{theorem}\label{thm:fullregret}
Let $W$ be a separable real inner-product space with corresponding norm $\|\cdot\|$ and suppose (with mild abuse of notation) every loss function $\ell_t:W\to \R$ has some subgradient $g_t\in W^*$ at $w_t$ such that $g_t(w) =  g_t\cdot w$ for some $g_t\in W$. Let $M_{\max} = \max_t M_t$. Then if $ \Lm=\max_t \|g_t\|$ and $L(t)=\max_{t'<t}\|g_t\|$, \algname\ achieves regret:
\begin{align*}
    R_T(u)&\le (2\psi(u)+96)\left(\log_2\left(\frac{\Lm}{L_1}\right)+1\right)\sqrt{M_{\max}+\|g\|^2_{1:T}}\\
    &\quad\quad+8\Lm\left(\log_2\left(\frac{\Lm}{L_1}\right)+1\right)\min\left[\exp\left(8\max_t \frac{\|g_t\|^2}{L(t)^2}\right),\exp(\sqrt{T/2})\right]\\
    &=O\left(\Lm\log\left(\frac{\Lm}{L_1}\right)\left[(\|u\|\log(\|u\|)+2)\sqrt{T}+\exp\left(8\max_t \frac{\|g_t\|^2}{L(t)^2}\right)\right]\right)
\end{align*}
\end{theorem}

The conditions on $W$ in Theorem \ref{thm:fullregret} are fairly mild. In particular they are satisfied whenever $W$ is finite-dimensional and in most kernel method settings \citep{hofmann2008kernel}. In the kernel method setting, $W$ is an RKHS of functions $\mathcal{X}\to \R$ and our losses take the form $\ell_t(w) = \ell_t(\langle w,k_{x_t}\rangle)$ where $k_{x_t}$ is the representing element in $W$ of some $x_t\in \mathcal{X}$, so that $g_t=\overline g_t k_{x_t}$ where $\overline g_t\in \partial \ell_t(\langle w,k_{x_t}\rangle)$.

Although we nearly match our lower-bound exponential term of $\exp((2T)^{1/2-\epsilon})$, in order to have a practical algorithm we need to do much better. Fortunately, the $\max_t \tfrac{\|g_t\|^2}{L(t)^2}$ term may be significantly smaller when the losses are not fully adversarial. For example, if the loss vectors $g_t$ satisfy $\|g_t\| = t^2$, then the exponential term in our bound reduces to a manageable constant even though $\|g_t\|$ is growing quickly without bound.

To prove Theorem \ref{thm:fullregret}, we bound the regret of \algname\ during each epoch. Recall that during an epoch, \algname\ is running FTRL with  $\psi_t(w)=\psi(w)/\eta_t$. Therefore our first order of business is to analyze the regret of FTRL across one of these epochs, which we do in Lemma \ref{thm:firstregret} (proved in appendix):
\begin{restatable}{lemma}{firstregret}\label{thm:firstregret}
Set $k=\sqrt{2}$. Suppose $\|g_t\|\le L$ for $t<T$, $1/L\le p \le2/L$, $g_T\le \Lm$ and $\Lm\ge L$. Let $W_{\max} = \max_{t\in[1,T]}\|w_t\|$. Then the regret of FTRL with regularizers $\psi_t(w)=\psi(w)/\eta_t$ is:
\begin{align*}
    R_T(u)&\le \psi(u)/\eta_T + 96\sqrt{M_T+\|g\|^2_{1:T}}+2\Lm\min\left[W_{\max},4\exp\left(4\frac{\Lm^2}{L^2}\right),\exp(\sqrt{T/2})\right]\\
    &\le (2\psi(u)+96)\sqrt{\sum_{t=1}^{T-1} L|g_t|+\Lm^2}+8\Lm\min\left[\exp\left(\frac{4\Lm^2}{L^2}\right),\exp(\sqrt{T/2})\right]\\
    &\le \Lm (2((\|u\|+1)\log(\|u\|+1)-\|u\|)+96)\sqrt{T}+8\Lm\min\left[e^{\frac{4\Lm^2}{L^2}},e^{\sqrt{T/2}}\right]
\end{align*}
\end{restatable}
Lemma \ref{thm:firstregret} requires us to know the value of $L$ in order to set $p$. However, the crucial point is that it encompasses the case in which $L$ is misspecified on the last loss vector. This allows us to show that \algname\ does not suffer too much by updating $p$ on-the-fly.

\begin{proof}[Proof of Theorem \ref{thm:fullregret}]

The theorem follows by applying Lemma \ref{thm:firstregret} to each epoch in which $L_t$ is constant.

Let $1=t_1,t_2,t_3,\cdots,t_{n}$ be the various increasing values of $t_\star$ (as defined in Algorithm \ref{alg:algname}), and we define $t_{n+1}=T+1$. Then define
\[
R_{a:b}(u)=\sum_{t=a}^{b-1} g_t(w_t-u)
\]

so that $R_T(u)\le\sum_{j=1}^{n} R_{t_j:t_{j+1}}(u)$. We will bound $R_{t_j:t_{j+1}}(u)$ for each $j$.

Fix a particular $j<n$.
Then $R_{t_j:t_{j+1}}(u)$ is simply the regret of FTRL with $k=\sqrt{2}$, $p=\frac{1}{L_{t_j}}$, $\eta_t=\frac{1}{k\sqrt{2(M_t+\|g\|^2_{t_{j}:t})}}$ and regularizers $\psi(w)/\eta_t$. By definition of $L_t$, for $t\in [1,t_{j+1}-2]$ we have $\|g_t\|\le 2L_{t_j}$. Further, if $L=\max_{t\in[1,t_{j+1}-2]}\|g_t\|$ we have $L\ge L_{t_j}$. Therefore, $L_{t_j}\le L\le 2L_{t_j}$ so that $\frac{1}{L}\le p\le \frac{2}{L}$. Further, we have $\|g_{t_{j+1}-1}\|/L_{t_j}\le 2\max_t \|g_t\|/L(t)$. Thus by Lemma \ref{thm:firstregret} we have
\begin{align*}
    R_{t_j:t_{j+1}}(u) &\le  \psi(u)/\eta_{t_{j+1}-1}+96\sqrt{M_{t_{j+1}-1}+\|g\|^2_{t_j:t_{j+1}-1}}\\
    &\quad\quad\quad+2\Lm \min\left[W_{\max},4\exp\left(4\frac{\|g_{t_{j+1}-1}\|^2}{L_{t_j}^2}\right),\exp\left(\frac{\sqrt{t_{j+1}-t_{j}}}{\sqrt{2}}\right)\right]\\
    &\le \psi(u)/\eta_{t_{j+1}-1}+96\sqrt{M_{\max}+\|g\|^2_{t_j:t_{j+1}-1}}+8\Lm\min\left[e^{8\max_t\frac{\|g_t\|^2}{L(t)^2}},e^{\sqrt{T/2}}\right]\\
    &\le (2\psi(u)+96)\sqrt{M_{\max}+\|g\|^2_{1:T}} +8\Lm\min\left[\exp\left(8\max_t\frac{\|g_t\|^2}{L(t)^2}\right),\exp(\sqrt{T/2})\right]
\end{align*}
Summing across epochs, we have
\begin{align*}
    R_T(u)&=\sum_{j=1}^n R_{t_j:t_{j+1}}(u)\\
    &\le n\left[(2\psi(u)+96)\sqrt{M_{\max}+\|g\|^2_{1:T}} +8\Lm\min\left[\exp\left(8\max_t\frac{\|g_t\|^2}{L(t)^2}\right),\exp\left(\sqrt{T/2}\right)\right]\right]\\
\end{align*}
Observe that $n\le \log_2(\Lm/L_1)+1$ to prove the first line of the theorem. The big-Oh expression follows from the inequality: $M_{t_{j+1}-1}\le L_{t_j}\sum_{t=t_j}^{t_{j+1}-1} \|g_t\|\le \Lm\sum_{t=1}^T \|g_t\|$.
\end{proof}

Our specific choices for $k$ and $p$ are somewhat arbitrary. We suspect (although we do not prove) that the preceding theorems are true for larger values of $k$ and any $p$ inversely proportional to $L_t$, albeit with differing constants. In Section \ref{sec:experiments} we perform experiments using the values for $k$, $p$ and $L_t$ described in Algorithm \ref{alg:algname}. In keeping with the spirit of designing a hyperparameter-free algorithm, no attempt was made to empirically optimize these values at any time.

\section{Experiments}\label{sec:experiments}
\subsection{Linear Classification}
To validate our theoretical results in practice, we evaluated \algname\ on 8 classification datasets. The data for each task was pulled from the libsvm website \citep{chang2011libsvm}, and can be found individually in a variety of sources \citep{guyon2004result,chang2001ijcnn,lecun1998gradient,lewis2004rcv1,duarte2004vehicle,Lichman:2013,kogan2009predicting}. We use linear classifiers with hinge-loss for each task and we compare \algname\ to five other optimization algorithms: \adagrad\ \citep{duchi10adagrad}, \scaleinvariant\ \citep{orabona2014generalized}, \pistol\ \citep{orabona2014simultaneous}, \adam\ \citep{kingma2014adam}, and \adadelta\ \citep{zeiler2012adadelta}. Each of these algorithms requires tuning of some hyperparameter for unconstrained problems with unknown $\Lm$ (usually a scale-factor on a learning rate). In contrast, our \algname\ requires no such tuning.

We evaluate each algorithm with the average loss after one pass through the data, computing a prediction, an error, and an update to model parameters for each example in the dataset. Note that this is not the same as a cross-validated error, but is closer to the notion of regret addressed in our theorems. We plot this average loss versus hyperparameter setting for each dataset in Figures \ref{fig:lossplots1} and \ref{fig:lossplots2}. These data bear out the effectiveness of \algname: while it is not unilaterally the highest performer on all datasets, it shows remarkable robustness across datasets with zero manual tuning.
\begin{figure*}[th]
\centering
\begin{minipage}{0.49\textwidth}
\subfigure{
\includegraphics[width = 1.0\textwidth]{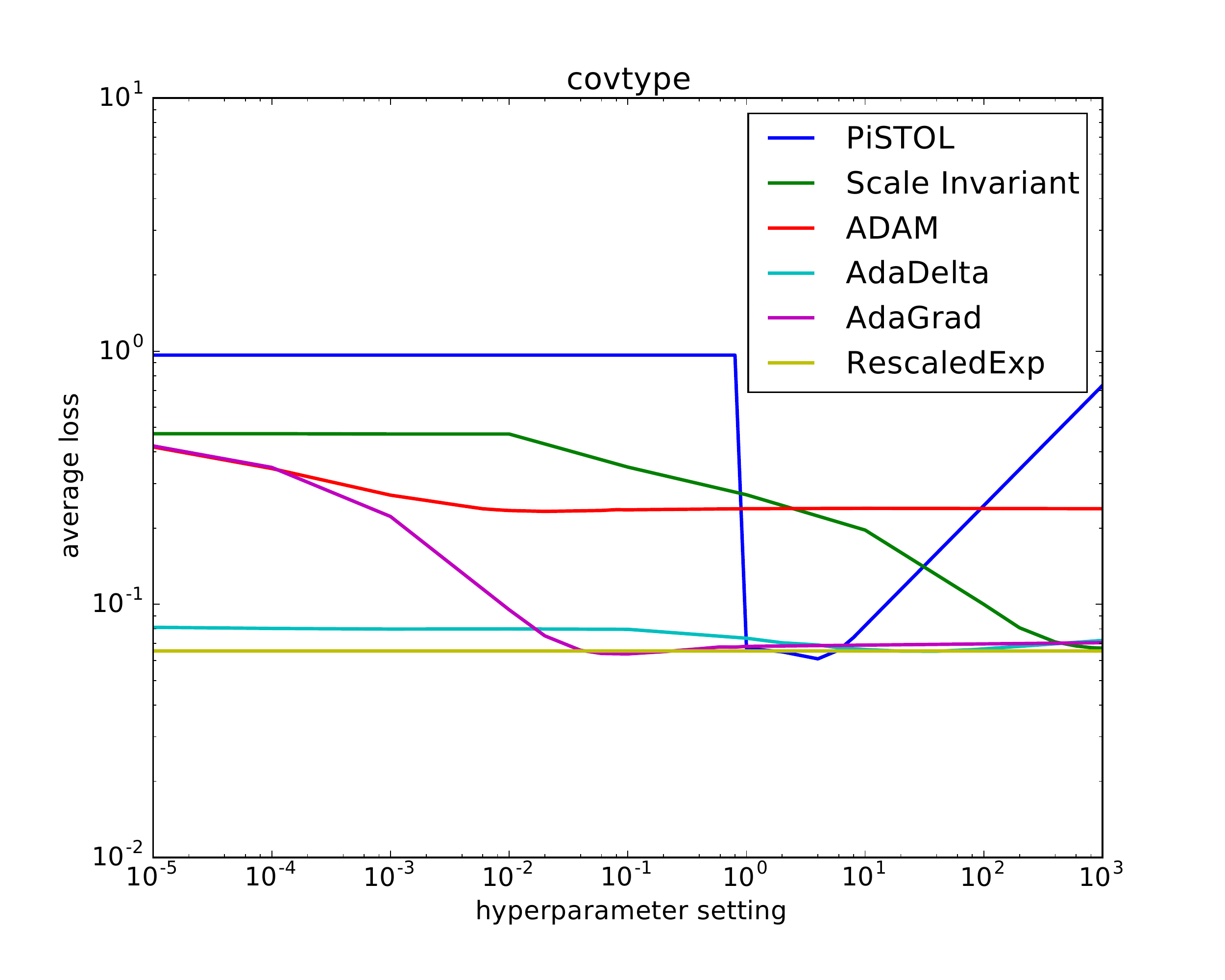}
}
\end{minipage}
\begin{minipage}{0.49\textwidth}
\subfigure{
\includegraphics[width = 1.0\textwidth]{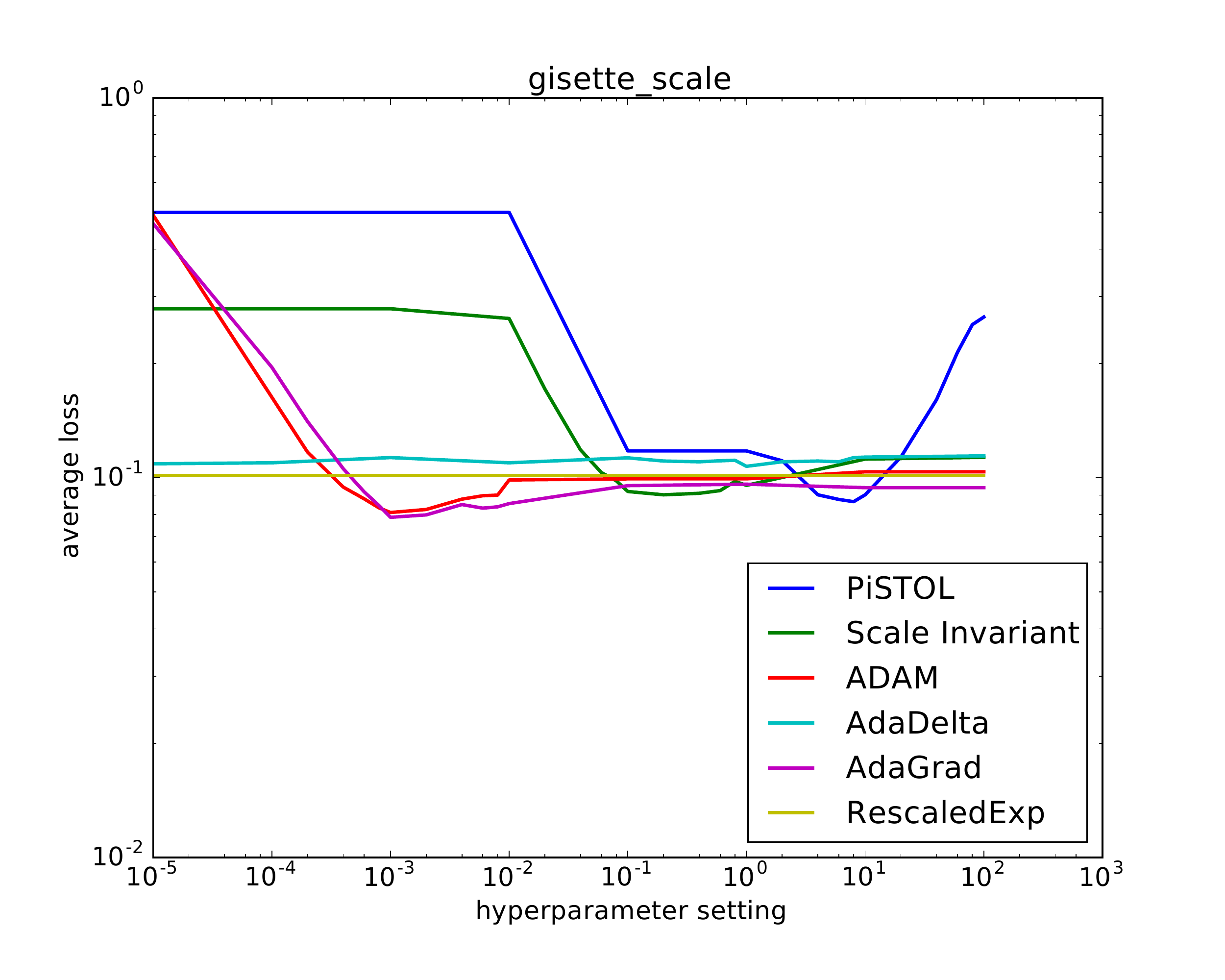}
}
\end{minipage}
\begin{minipage}{0.49\textwidth}
\subfigure{
\includegraphics[width = 1.0\textwidth]{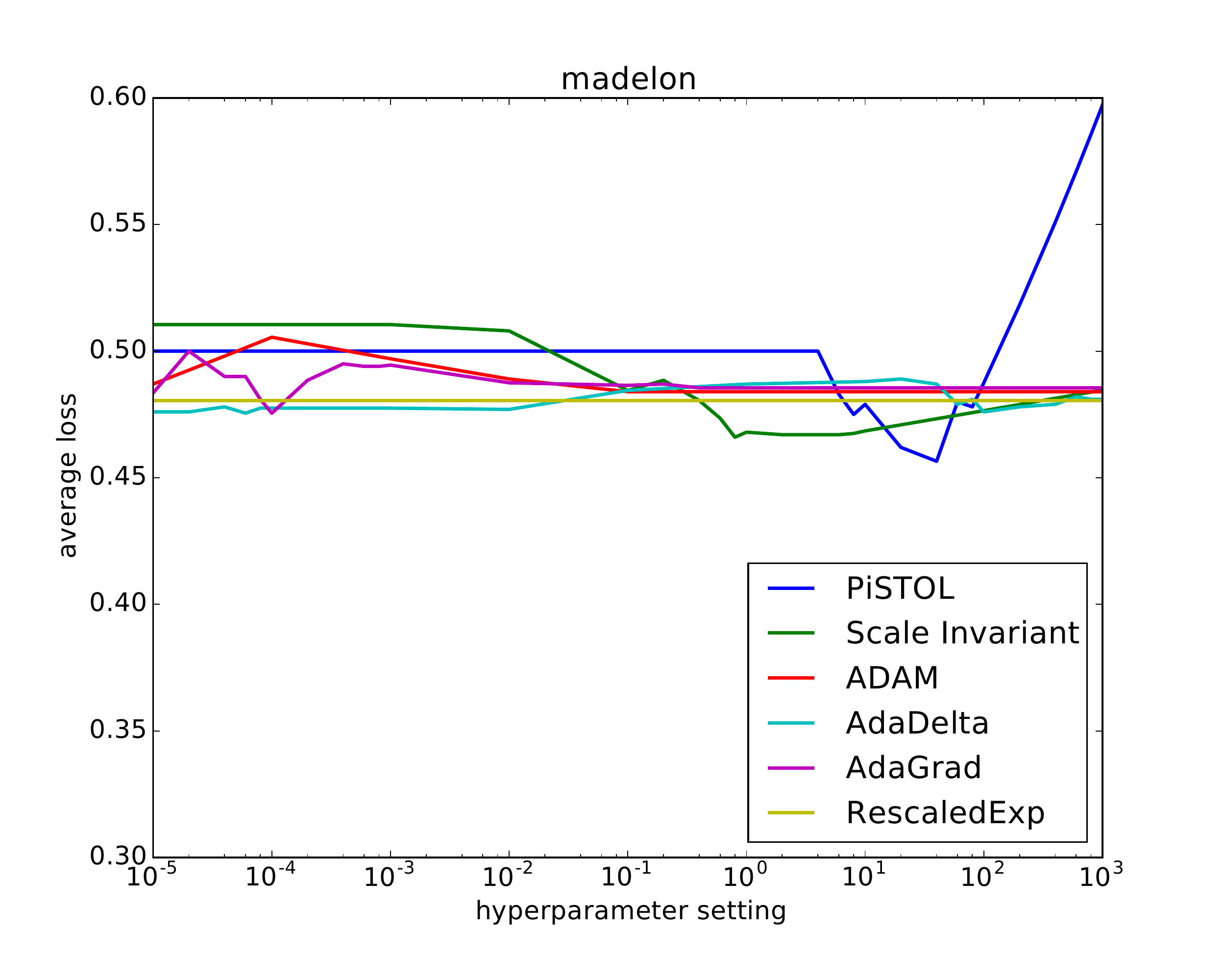}
}
\end{minipage}
\begin{minipage}{0.49\textwidth}
\subfigure{
\includegraphics[width = 1.0\textwidth]{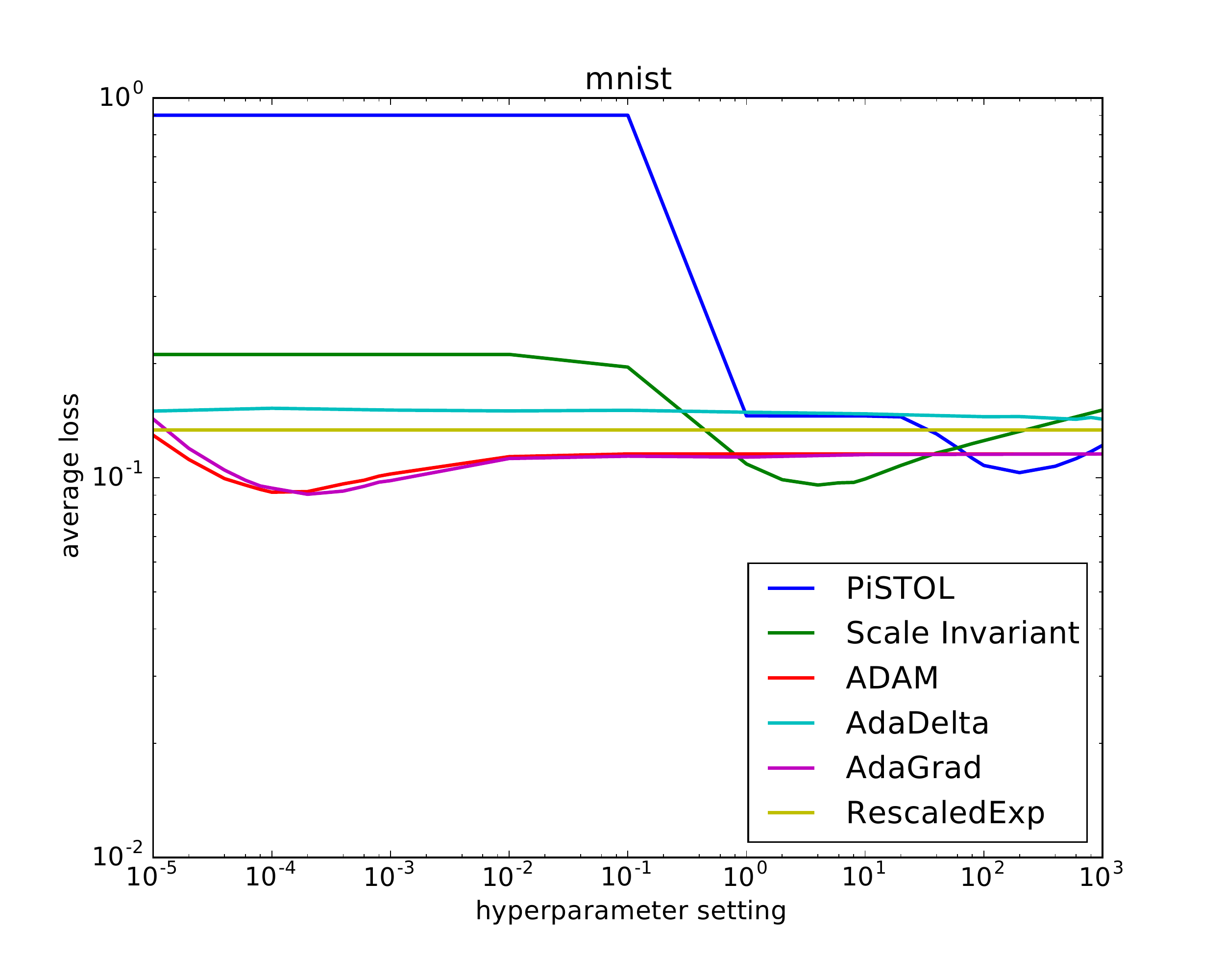}
}
\end{minipage}
\caption{Average loss vs hyperparameter setting for each algorithm across each dataset. \algname\ has no hyperparameters and so is represented by a flat yellow line. Many of the other algorithms display large sensitivity to hyperparameter setting.
}
\label{fig:lossplots1}
\end{figure*}
\begin{figure*}[th]
\centering
\begin{minipage}{0.49\textwidth}
\subfigure{
\includegraphics[width =1.0\textwidth]{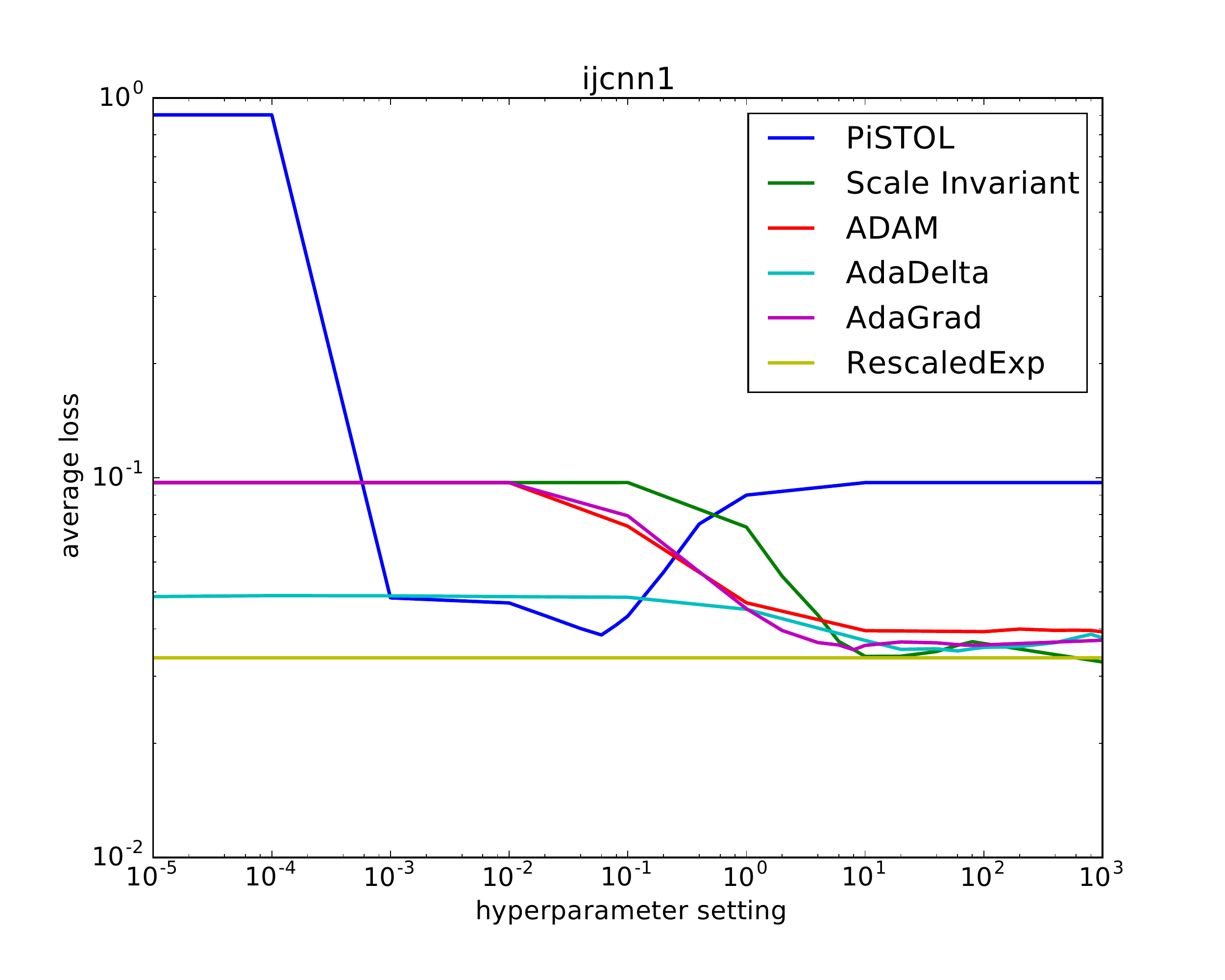}
}
\end{minipage}
\begin{minipage}{0.49\textwidth}
\subfigure{
\includegraphics[width = 1.0\textwidth]{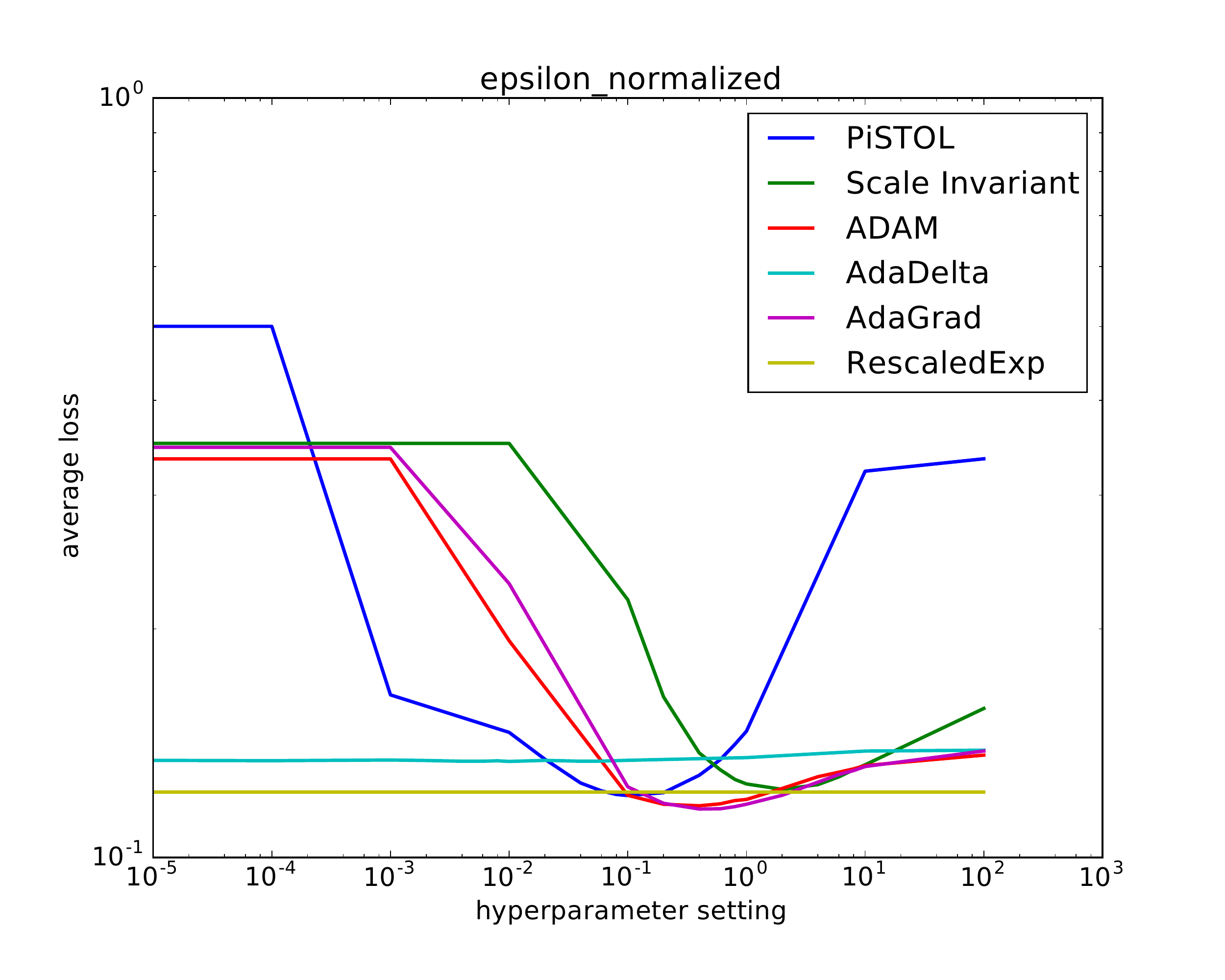}
}
\end{minipage}
\begin{minipage}{0.49\textwidth}
\subfigure{
\includegraphics[width = 1.0\textwidth]{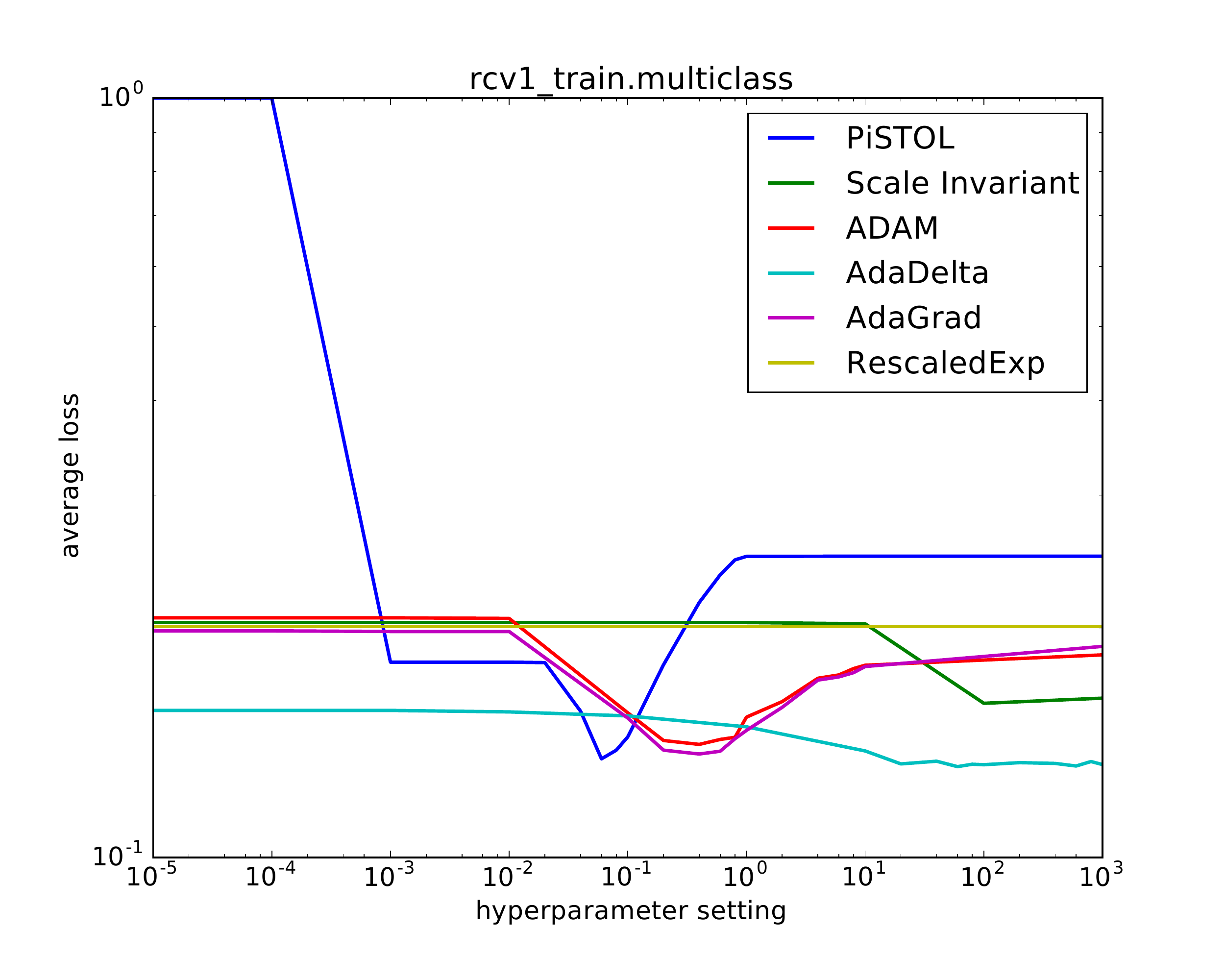}

}
\end{minipage}
\begin{minipage}{0.49\textwidth}
\subfigure{
\includegraphics[width = 1.0\textwidth]{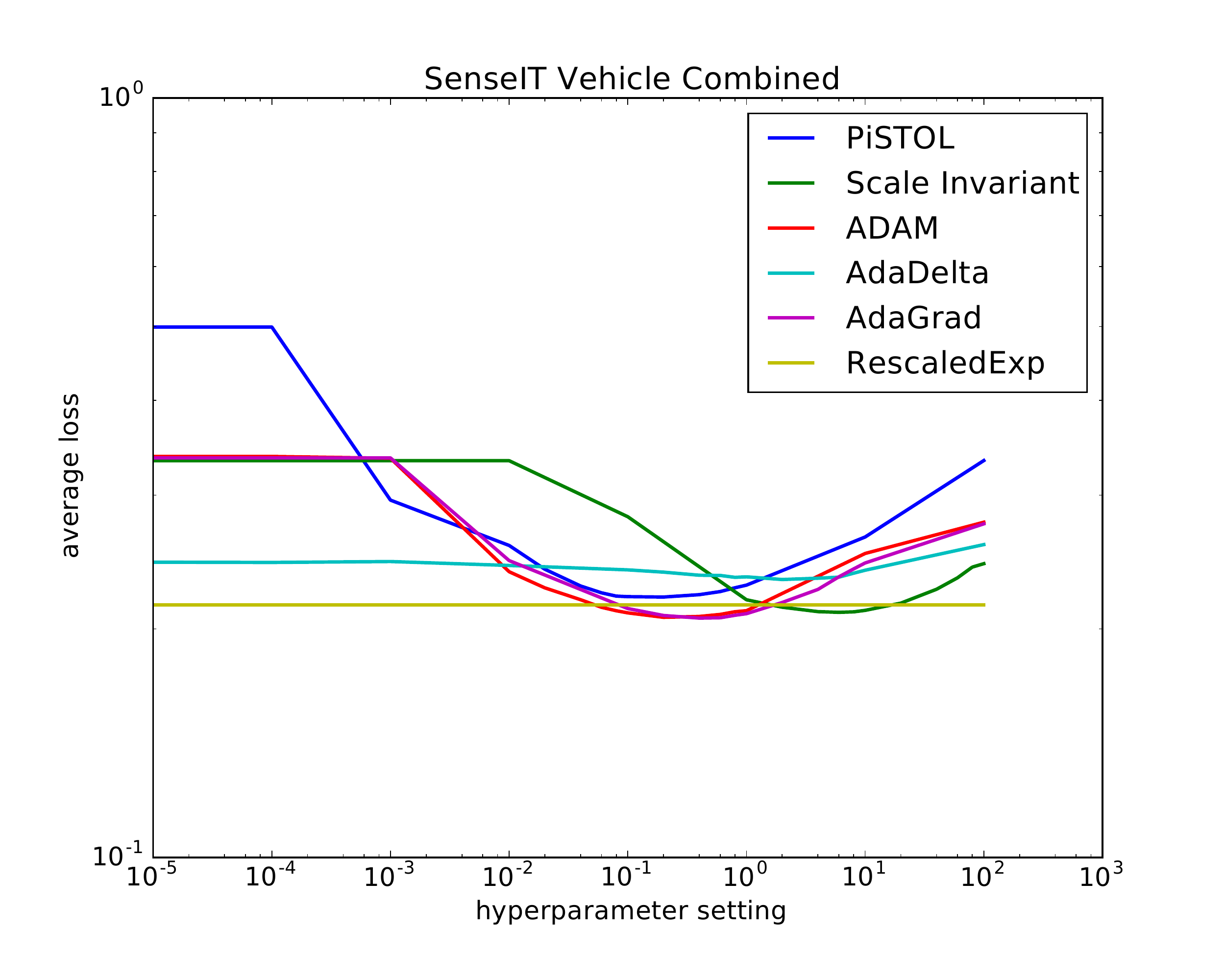}
}
\end{minipage}
\caption{Average loss vs hyperparameter setting, continued from Figure \ref{fig:lossplots1}.
}
\label{fig:lossplots2}
\end{figure*}
\subsection{Convolutional Neural Networks}
We also evaluated \algname\ on two convolutional neural network models. These models have demonstrated remarkable success in computer vision tasks and are becoming increasingly more popular in a variety of areas, but can require significant hyperparameter tuning to train. We consider the MNIST \citep{lecun1998gradient} and CIFAR-10 \citep{krizhevsky2009learning} image classification tasks.

Our MNIST architecture consisted of two consecutive $5\times 5$ convolution and $2\times 2$ max-pooling layers followed by a 512-neuron fully-connected layer. Our CIFAR-10 architecture was two consecutive $5\times 5$ convolution and $3\times 3$ max-pooling layers followed by a $384$-neuron fully-connected layer and a $192$-neuron fully-connected layer. 
\begin{figure*}[th!]
\centering
\begin{minipage}{0.49\textwidth}
\subfigure{
\includegraphics[width =1.0\textwidth]{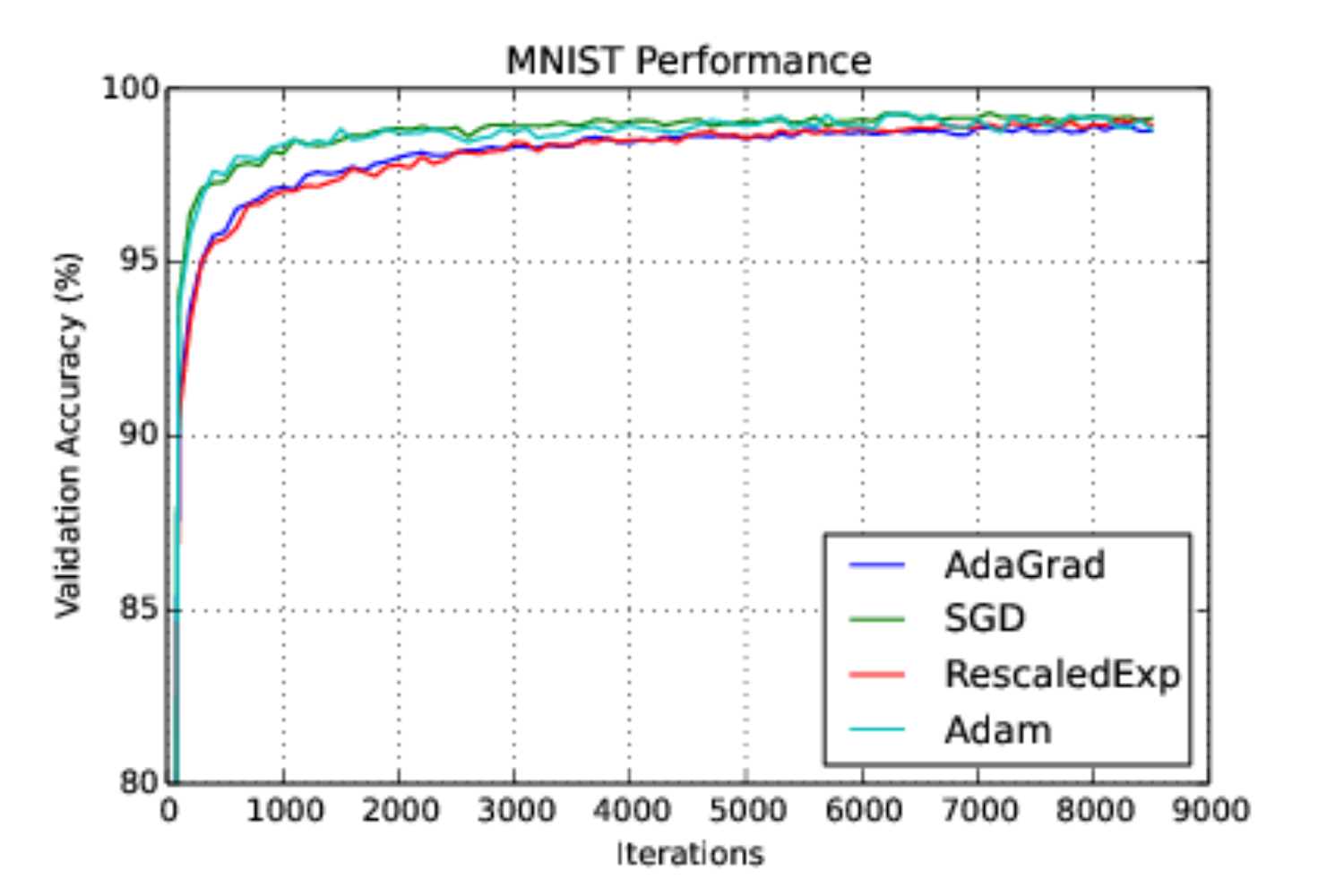}
}
\end{minipage}
\begin{minipage}{0.49\textwidth}
\subfigure{
\includegraphics[width = 1.0\textwidth]{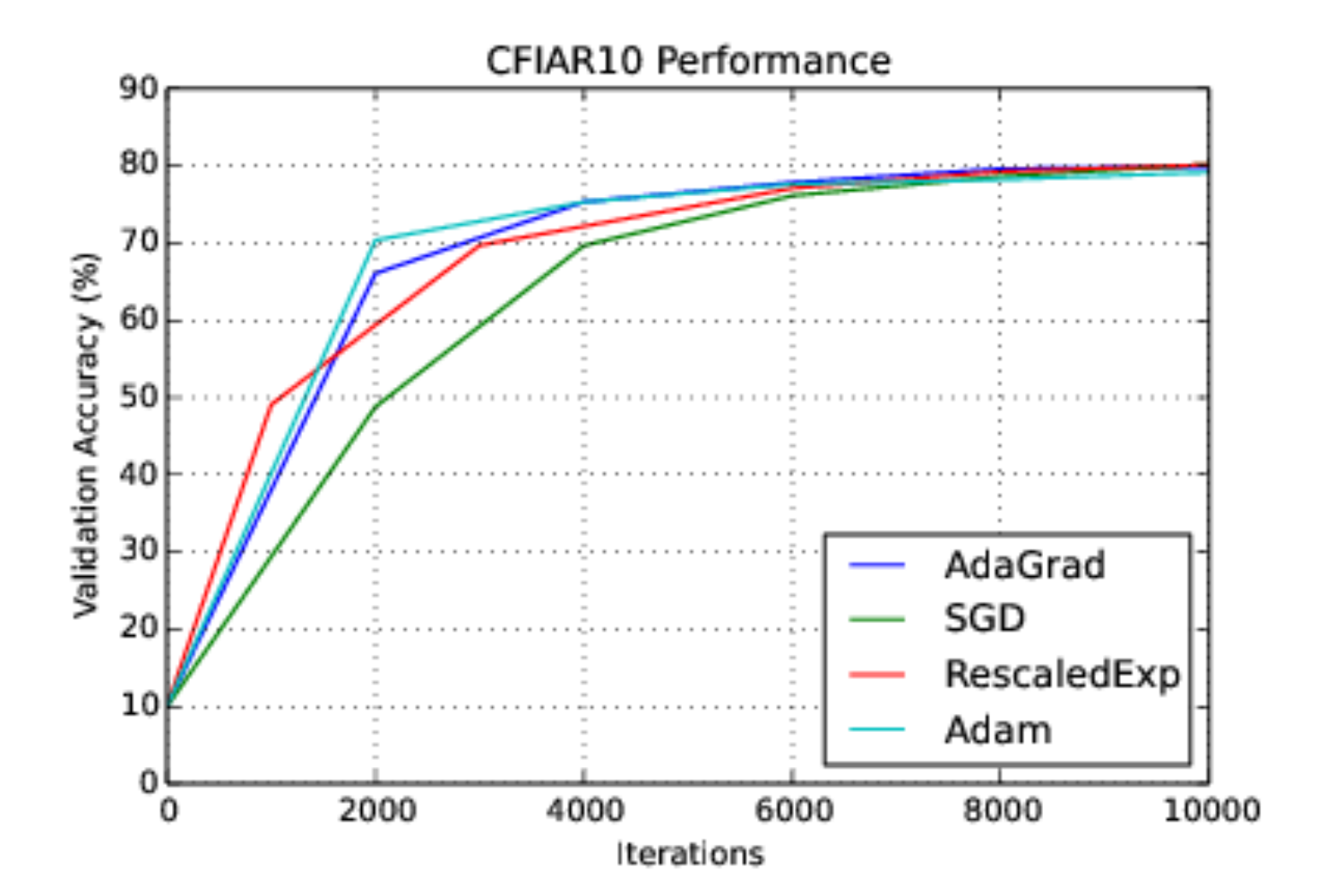}

}
\end{minipage}
\caption{We compare \algname\ to \adam, \adagrad, and stochastic gradient descent (SGD), with learning-rate hyperparameter optimization for the latter three algorithms. All algorithms achieve a final validation accuracy of $99\%$ on MNIST and $84\%$, $84\%$, $83\%$ and $85\%$ respectively on CIFAR-10 (after 40000 iterations).
}
\label{fig:cnn}
\end{figure*}

These models are highly non-convex, so that none of our theoretical analysis applies. Our use of \algname\ is motivated by the fact that in practice convex methods are used to train these models. We found that \algname\ can match the performance of other popular algorithms (see Figure \ref{fig:cnn}). 

In order to achieve this performance, we made a slight modification to \algname: when we update $L_t$, instead of resetting $w_t$ to zero, we re-center the algorithm about the previous prediction point. We provide no theoretical justification for this modification, but only note that it makes intuitive sense in stochastic optimization problems, where one can reasonably expect that the previous prediction vector is closer to the optimal value than zero.

\section{Conclusions}\label{sec:conclusions}

We have presented \algname, an Online Convex Optimization algorithm that achieves regret $\tilde O(\|u\|\log(\|u\|)\Lm\sqrt{T}+\exp(8\max_t \|g_t\|^2/L(t)^2))$ where $\Lm=\max_t \|g_t\|$ is unknown in advance. Since \algname\ does not use any prior-knowledge about the losses or comparison vector $u$, it is hyperparameter free and so does not require any tuning of learning rates. We also prove a lower-bound showing that any algorithm that addresses the unknown-$\Lm$ scenario must suffer an exponential penalty in the regret. We compare \algname\ to prior optimization algorithms empirically and show that it matches their performance.

While our lower-bound matches our regret bound for \algname\ in terms of $T$, clearly there is much work to be done. 
For example, when \algname\ is run on the adversarial loss sequence presented in Theorem \ref{thm:lowerbound}, its regret matches the lower-bound, suggesting that the optimality gap could be improved with superior analysis.
We also hope that our lower-bound inspires work in algorithms that adapt to non-adversarial properties of the losses to avoid the exponential penalty.
\small
\bibliographystyle{unsrt}
\bibliography{all}

\appendix

\section{Follow-the-Regularized-Leader (FTRL) Regret}
Recall that the FTRL algorithm uses the strategy $w_{t+1} = \argmin \psi_t(w) + \sum_{t'=1}^t \ell_{t'}(w)$, where the functions $\psi_t$ are called regularizers.
\begin{theorem}\label{thm:ftrlregret}
FTRL with regularizers $\psi_t$ and $\psi_0(w_1)=0$ obtains regret:
\begin{align}\label{eqn:general_regret}
    R_t(u)&\le\psi_T(u)+\sum_{t=1}^T \psi_{t-1}(w_{t+1})-\psi_t(w_{t+1})+\ell_t(w_t)-\ell_t(w_{t+1})
\end{align}
Further, if the losses are linear $\ell_t(w)=g_t\cdot w$ and $\psi_t(w)=\frac{1}{\eta_t}\psi(w)$ for some values $\eta_t$ and fixed function $\psi$, then the regret is
\begin{align}\label{eqn:linear_regret}
    R_t(u)&\le\frac{1}{\eta_T}\psi(u)+\sum_{t=1}^T\left(\frac{1}{\eta_{t-1}}-\frac{1}{\eta_t}\right)\psi(w_{t+1})+g_t\cdot(w_t-w_{t+1}) 
\end{align}
\end{theorem}
\begin{proof}
The first part follows from some algebraic manipulations:
\begin{align*}
\sum_{t=1}^T \ell_t(u)+\psi_T(u)&\ge \psi_T(w_{T+1}) +\sum_{t=1}^T \ell_t(w_{T+1})\\
-\sum_{t=1}^T\ell_t(u)&\le \psi_T(u)-\psi_T(w_{T+1})-\sum_{t=1}^T \ell_t(w_{T+1})\\
\end{align*}
\begin{align*}
R_T(u)&=\sum_{t=1}^T \ell_t(w_t)-\sum_{t=1}^T\ell_t(u)\\
&\le \psi_T(u)-\psi_T(w_{T+1})+\sum_{t=1}^T \ell_t(w_t)-\ell_t(w_{T+1})\\
&=\psi_T(u)-\psi_T(w_{T+1})+\ell_T(w_{T})-\ell_T(w_{T+1}) +R_{T-1}(w_{T+1})\\
&\le\psi_T(u)-\psi_T(w_{T+1})+\ell_T(w_{T})-\ell_T(w_{T+1})+\\
&\ +\sum_{t=1}^{T-1} \psi_{t}(w_{t+2})-\psi_t(w_{t+1})+\ell_t(w_t)-\ell_t(w_{t+1})\\
&=\psi_T(u)+\ell_1(w_1)-\ell_1(w_2)-\psi_1(w_2)\\
&\ +\sum_{t=2}^{T}  \psi_{t-1}(w_{t+1})-\psi_t(w_{t+1})+\ell_t(w_t)-\ell_t(w_{t+1})\\
&=\psi_T(u)+\sum_{t=1}^T \psi_{t-1}(w_{t+1})-\psi_t(w_{t+1})+\ell_t(w_t)-\ell_t(w_{t+1})
\end{align*}

where we're assuming $\psi_0(w_1)=0$ in the last step.

Now let's specialize to the case of linear losses $\ell_t(w)=g_t\cdot w$ and regularizers of the form $\psi_t(w)=\frac{1}{\eta_t} \psi(w)$ for some fixed regularizer $\psi$ and varying scalings $\eta_t$. Plugging this into the previous bound gives:
\begin{align*}
    R_t(u)\le \frac{1}{\eta_T}\psi(u)+\sum_{t=1}^T\left(\frac{1}{\eta_{t-1}}-\frac{1}{\eta_t}\right)\psi(w_{t+1})+g_t\cdot(w_t-w_{t+1}) 
\end{align*}
\end{proof}

While this formulation of the regret of FTRL is sufficient for our needs, our analysis is not tight. We refer the reader to \citep{mcmahan2014survey} for a stronger FTRL bound that can improve constants in some analyses.

\section{Proof of Lemma \ref{thm:firstregret}}

We start off by computing the FTRL updates with regularizers $\psi(w)/\eta_t$:
\[
\nabla \psi(w) = \log(\|w\|+1)\frac{w}{\|w\|}
\]
so that
\begin{align*}
w_{T+1} &= \argmin \frac{1}{\eta_T}\psi(w)+\sum_{t=1}^T g_t\cdot w\\
&=-\frac{g_{1:t}}{\|g_{1:t}\|}(\exp(\eta_T\|g_{1:T}\|)-1)
\end{align*}

Our goal will be to show that the terms $\left(\frac{1}{\eta_{t-1}}-\frac{1}{\eta_t}\right)\psi(w_{t+1})+g_t\cdot(w_t-w_{t+1})$ in the sum in (\ref{eqn:linear_regret}) are negative. In particular, note that sequence of $\eta_t$ is non-increasing so that $\left(\frac{1}{\eta_{t-1}}-\frac{1}{\eta_t}\right)\psi(w_{t+1})\le 0$ for all $t$. Thus our strategy will be to bound $g_t\cdot(w_t-w_{t+1})$.

\subsection{Reduction to one dimension}\label{sec:dimfree}

In order to bound $\left(\frac{1}{\eta_{t-1}}-\frac{1}{\eta_t}\right)\psi(w_{t+1})+g_t\cdot(w_t-w_{t+1})$, we first show that it suffices to consider the case when $g_t$ and $g_{1:t-1}$ are co-linear.

\begin{theorem}\label{thm:onedimensionreduction}
Let $W$ be a separable inner-product space and suppose (with mild abuse of notation) every loss function $\ell_t:W\to \R$ has some subgradient $g_t\in W^*$ such that $g_tw = \langle g_t,w\rangle$ for some $g_t\in W$. 
Suppose we run an FTRL algorithm with regularizers $\frac{1}{\eta_t}\psi(\|w\|)$ on loss functions $\ell_t$ such that $w_{t+1}=\frac{g_{1:t}}{\|g_{1:t}\|}f(\eta_t\|g_{1:t}\|)$ for some function $f$ for all $t$ where $\eta_t = \frac{c}{\sqrt{M_t+\|g\|^2_{1:t}}}$ for some constant $c$. Then for any $g_t$ with $\|g_t\|=L$, both $(\eta_{t-1}^{-1}-\eta_t^{-1})\psi(\|w_{t+1}\|)+g_t(w_t-w_{t+1})$ and $g_t(w_t-w_{t+1})$ are maximized when $g_t$ is a scalar multiple of $g_{1:t-1}$.
\end{theorem}
\begin{proof}

The proof is an application of Lagrange multipliers. Our Lagrangian for $(\eta_{t-1}^{-1}-\eta_t^{-1})\psi(\|w_{t+1}\|)+g_t(w_t-w_{t+1})$ is
\begin{align*}
    \mathcal{L} &= (\eta_{t-1}^{-1}-\eta_t^{-1})\psi(\|w_{t+1}\|)+g_t(w_t-w_{t+1})+\lambda \|g_t\|^2/2\\
    &=(\eta_{t-1}^{-1}-\eta_t^{-1})\psi(f(\eta_t\|g_{1:t}\|))+g_t\left(w_t-\frac{g_{1:t}}{\|g_{1:t}\|}f(\eta_t\|g_{1:t}\|)\right)+\lambda \frac{\|g_t\|^2}{2}
\end{align*}

Fix a countable orthonormal basis of $W$. For a vector $v\in W$ we let $v_i$ be the projection of $v$ along the $i$th basis vector of our countable orthonormal basis. We denote the action of $\nabla \mathcal{L}$ on the $i$th basis vector by $\nabla\mathcal{L}_i$.

Then we have
\begin{align*}
    \nabla\mathcal{L}_i&=\lambda g_{t,i} + w_{t,i}-w_{t+1,i}-\frac{g_{t,i}}{\|g_{1:t}\|}f(\eta_t\|g_{1:t}\|)\\
    &\quad\quad+\sum_j \frac{g_{t,j}(g_{1:t})_j}{\|g_{1:t}\|^3}(g_{1:t})_i f(\eta_t\|g_{1:t}\|)\\
    &\quad\quad - \sum_j \frac{(g_{1:t})_jg_{t,j}}{\|g_{1:t}\|}f'(\eta_t\|g_{1:t}\|)\left[\frac{(g_{1:t})_i\eta_t}{\|g_{1:t}\|}-\frac{\|g_{1:t}\| c \left(\frac{\partial M_t}{\partial g_{t,i}}+2g_{t,i}\right)}{2(M_t+\|g\|^2_{1:t})^{3/2}}\right]\\
    &\quad\quad+(\eta_{t-1}^{-1}-\eta_t^{-1})\psi'(f(\eta_t\|g_{1:t}\|))f'(\eta_t\|g_{1:t}\|)\left[\frac{(g_{1:t})_i\eta_t}{\|g_{1:t}\|}-\frac{\|g_{1:t}\| c \left(\frac{\partial M_t}{\partial g_{t,i}}+2g_{t,i}\right)}{2(M_t+\|g\|^2_{1:t})^{3/2}}\right]\\
    &\quad\quad-\psi(f(\eta_t\|g_{1:t}\|))\frac{\frac{\partial M_t}{\partial g_{t,i}}+2g_{t,i}}{2c\sqrt{M_t+\|g\|^2_{1:t}}}\\
    &=\lambda g_{t,i}+w_{t,i}-w_{t+1,i}+Ag_{t,i}+B(g_{1:t-1})_i+C\frac{\partial M_{t}}{\partial g_{t,i}}
\end{align*}

where $A$, $B$ and $C$ do not depend on $i$. Since $w_{t,i}$ and $w_{t+1,i}$ are scalar multiples of $g_{1:t-1}$ and $g_{1:t}$ respectively, we can reassign the variables $A$ and $B$ to write
\begin{align*}
    \nabla\mathcal{L}_i&=Ag_{t,i}+B(g_{1:t-1})_i+C\frac{\partial M_{t}}{\partial g_{t,i}}
\end{align*}

Now we compute
\begin{align*}
    \frac{\partial M_{t}}{\partial g_{t,i}}&=\frac{\partial \max(M_{t-1},\|g_{1:t}\|/p-\|g\|^2_{1:t})}{\partial g_{t,i}}\\
    &=\left\{\begin{array}{lr}0&:M_t = M_{t-1}\\
    \frac{(g_{1:t})_i}{p\|g_{1:t}\|}-2g_{t,i}&:M_t\ne M_{t-1}
    \end{array}\right.
\end{align*}
Thus after again reassigning the variables $A$ and $B$ we have
\begin{align*}
    \nabla\mathcal{L}_i&=Ag_{t,i}+B(g_{1:t-1})_i
\end{align*}

Therefore we can only have $\nabla \mathcal{L}=0$ if $g_t$ is a scalar multiple of $g_{1:t-1}$ as desired.

For $g_t(w_t-w_{t+1})$, we apply exactly the same argument. The Lagrangian is 
\begin{align*}
    \mathcal{L} &= g_t(w_t-w_{t+1})+\lambda \|g_t\|^2/2\\
    &=g_t\left(w_t-\frac{g_{1:t}}{\|g_{1:t}\|}f(\eta_t\|g_{1:t}\|)\right)+\lambda \frac{\|g_t\|^2}{2}
\end{align*}

and differentiating we have

\begin{align*}
    \nabla\mathcal{L}_i&=\lambda g_{t,i} + w_{t,i}-w_{t+1,i}-\frac{g_{t,i}}{\|g_{1:t}\|}f(\eta_t\|g_{1:t}\|)\\
    &\quad\quad+\sum_j \frac{g_{t,j}(g_{1:t})_j}{\|g_{1:t}\|^3}(g_{1:t})_i f(\eta_t\|g_{1:t}\|)\\
    &\quad\quad - \sum_j \frac{(g_{1:t})_jg_{t,j}}{\|g_{1:t}\|}f'(\eta_t\|g_{1:t}\|)\left[\frac{(g_{1:t})_i\eta_t}{\|g_{1:t}\|}-\frac{\|g_{1:t}\| c \left(\frac{\partial M_t}{\partial g_{t,i}}+2g_{t,i}\right)}{2(M_t+\|g\|^2_{1:t})^{3/2}}\right]\\
    &=\lambda g_{t,i}+w_{t,i}-w_{t+1,i}+Ag_{t,i}+B(g_{1:t-1})_i+C\frac{\partial M_{t}}{\partial g_{t,i}}\\
    &=Ag_{t,i}+B(g_{1:t-1})_i
\end{align*}

so that again we are done.

\end{proof}

We make the following intuitive definition:
\begin{definition}
For any vector $v\in W$, define $\sign(v)=\frac{v}{\|v\|}$.
\end{definition}
In the next section, we prove bounds on the quantity $(\eta_{t-1}^{-1}-\eta_t^{-1})\psi(\|w_{t+1}\|)+g_t(w_t-w_{t+1})$. By Theorem \ref{thm:onedimensionreduction} this quantity is maximized when $\sign(g_t)=\pm \sign(g_{1:t-1})$ and so we consider only this case.

\subsection{One dimensional FTRL}\label{1danalysis}

In this section we analyze the regret of our FTRL algorithm with the end-goal of proving Lemma \ref{thm:firstregret}. We make heavy use of Theorem \ref{thm:onedimensionreduction} to allow us to consider only the case $\sign(g_t)=\pm\sign(g_{1:t-1})$. In this setting we may identify the 1-dimensional space spanned by $g_t$ and $g_{1:t-1}$ with $\R$. Thus whenever we are operating under the assumption $\sign(g_t)=\sign(g_{1:t-1})$ we will use $|\cdot|$ in place of $\|\cdot\|$ and occasionally assume $g_{1:t-1}>0$ as this holds WLOG. We feel that this notation and assumption aids intuition in visualizing the following results.

\begin{lemma}\label{thm:diffbound}
Suppose $\sign(g_t)=\sign(g_{1:t-1})$. Then
\begin{equation}\label{eqn:tightdiffbound}
|\eta_{t-1}\|g_{1:t-1}\|-\eta_t\|g_{1:t}\||\le \eta_t\|g_t\|
\end{equation}
Suppose instead that $\sign(g_t)=-\sign(g_{1:t-1})$ and also $\|g_t\|\le L$. Then we still have:
\begin{equation}\label{eqn:diffbound}
|\eta_{t-1}\|g_{1:t-1}\|-\eta_t\|g_{1:t}\||\le \left(1+\frac{pL}{2}\right)\eta_t\|g_t\|
\end{equation}
\end{lemma}
\begin{proof}
First, suppose $\sign(g_t)=\sign(g_{1:t-1})$. Then $\sign(g_{1:t})=\sign(g_{1:t-1})$. WLOG, assume $g_{1:t-1}>0$. Notice that $\eta_tg_{1:t}$ is an increasing function of $g_t$ for $g_t>0$ because $\eta_tg_{1:t}$ is proportional to either $g_{1:t}$ or $\sqrt{g_{1:t}}$ depending on whether $M_t=M_{t-1}$ or not. Then since $\eta_t<\eta_{t-1}$ we have 
\begin{align*}
    |\eta_{t-1}g_{1:t-1}-\eta_tg_{1:t}|&=\eta_tg_{1:t}-\eta_{t-1}g_{1:t-1}\\
    &\le \eta_tg_{1:t}-\eta_tg_{1:t-1}\\
    &=\eta_t|g_t|
\end{align*}
so that (\ref{eqn:tightdiffbound}) holds.

Now suppose $\sign(g_t)=-\sign(g_{1:t-1})$ and $\|g_t\|\le L$. We consider two cases. 

\noindent{\bf Case 1: $\eta_t|g_{1:t}|\ge \eta_{t-1}|g_{1:t-1}|$:} 

Since $\eta_{t-1}\ge \eta_t$, we have
\begin{align*}
    \eta_t|g_{1:t}|&\ge \eta_{t-1}|g_{1:t-1}|\\
    \eta_t|g_{1:t}|&\ge \eta_t |g_{1:t-1}|\\
    |g_{1:t}|&\ge |g_{1:t-1}|\\
    |g_t|&\ge |g_{1:t}|
\end{align*} where the last line follows since $\sign(g_{1:t-1})=-\sign(g_t)$. Therefore:
\begin{align*}
    |\eta_{t-1}|g_{1:t-1}|-\eta_t|g_{1:t}||\le \eta_t|g_{1:t}|\le \eta_t|g_t|
\end{align*}
so that we are done. 

\noindent{\bf Case 2: $\eta_t|g_{1:t}|\le \eta_{t-1}|g_{1:t-1}|$:} 

When $g_t<-g_{1:t-1}$ and $\eta_t|g_{1:t}|\le \eta_{t-1}|g_{1:t-1}|$,
$|\eta_{t-1}|g_{1:t-1}|-\eta_t|g_{1:t}||$ is a decreasing function of $|g_t|$ because $\eta_t|g_{t:1}|$ is an increasing function of $|g_t|$ for $g_t<-g_{1:t-1}$.
Therefore it suffices to consider the case $g_t\ge -g_{1:t-1}$, so that $\sign(g_{1:t})=\sign(g_{1:t-1})$ and $|g_{1:t}|\le |g_{1:t-1}|$:

Since $|g_{1:t}|\le |g_{1:t-1}|$, we have $M_t=M_{t-1}$ so that we can write:
\begin{align*}
    \eta_{t-1}g_{1:t-1}-\eta_tg_{1:t}&=-g_t\eta_t+g_{1:t-1}(\eta_{t-1}-\eta_t)\\
    &=|g_t|\eta_t+g_{1:t-1}\left(\frac{1}{k\sqrt{2}\sqrt{M_{t-1}+\|g\|^2_{1:t-1}}}-\frac{1}{k\sqrt{2}\sqrt{M_t+\|g\|^2_{1:t-1}+g_t^2}}\right)\\
    &=|g_t|\eta_t+\frac{g_{1:t-1}}{k\sqrt{2}}\left(\frac{1}{\sqrt{M_{t-1}+\|g\|^2_{1:t-1}}}-\frac{1}{\sqrt{M_{t-1}+\|g\|^2_{1:t-1}+g_t^2}}\right)\\
    &\le|g_t|\eta_t+\frac{g_{1:t-1}}{k\sqrt{2}\sqrt{M_{t}+\|g\|^2_{1:t-1}+g^2_t}}\left(\frac{\sqrt{M_{t-1}+\|g\|^2_{1:t-1}+g^2_t}}{\sqrt{M_{t-1}+\|g\|^2_{1:t-1}}}-1\right)\\
    &\le|g_t|\eta_t+g_{1:t-1}\eta_t\left(1+\frac{g^2_t}{2(M_{t-1}+\|g\|^2_{1:t-1})}-1\right)\\
    &\le|g_t|\eta_t+\eta_t\frac{g_{1:t-1}g_t^2}{2(M_{t-1}+\|g\|^2_{1:t-1})}\\
    &\le|g_t|\eta_t(1+\frac{pL}{2})
\end{align*}
we have used the identity $\sqrt{X+g_t^2}\le\sqrt{X}+\frac{g_t^2}{2\sqrt{X}}$ between lines 4 and 5, and
the last line follows because $|g_t|\le L$ and $M_{t-1}+\|g\|^2_{1:t-1}\ge |g_{1:t-1}|/p$.
\end{proof}

\begin{lemma}\label{thm:wtogbound}
If 
\[
\|w_T\|\ge \exp\left(\frac{\sqrt{pB}}{k\sqrt{2}}\right)-1
\]
then
\[
\|g_{1:T-1}\|\ge B
\]
\end{lemma}
\begin{proof}
First note that by definition of $M_{T-1}$ and $\eta_{T-1}$, $\eta_{T-1}\|g_{1:T-1}\|\le \frac{\sqrt{p\|g_{1:T-1}\|}}{k\sqrt{2}}$. The proof now follows from some algebra:
\begin{align*}
     \exp\left(\frac{\sqrt{pB}}{k\sqrt{2}}\right)&\le \|w_T\|+1\\
     &=\exp(\eta_{T-1}\|g_{1:T-1}\|)\\
     &\le \exp\left(\frac{\sqrt{p\|g_{1:T-1}\|}}{k\sqrt{2}}\right)\\
\end{align*}
Taking squares of logs and rearranging now gives the desired inequality.
\end{proof}
We have the following immediate corollary:
\begin{corollary}\label{thm:wtosignbound}
Suppose $\sign(g_t)=\pm\sign(g_{1:t-1})$, $\|g_t\|\le L$, and 
\[
\|w_t\|\ge \exp\left(\frac{\sqrt{pL}}{k\sqrt{2}}\right)-1
\]
Then $\sign(g_{1:t})=\sign(g_{1:t-1})$.
\end{corollary}

Now we begin analysis of the sum term in (\ref{eqn:linear_regret}).

\begin{lemma}\label{thm:stability}
Suppose $\sign(g_{1:t})=\sign(g_{1:t-1})$ and $|g_t|\le L$. Then
\[
|w_t-w_{t+1}|\le |g_t|\eta_t(|w_{t+1}|+1)\left(1+\frac{pL}{2}\right)\exp\left[g_t\eta_t\left(1+\frac{pL}{2}\right)\right]
\]
\end{lemma}
\begin{proof}
Since $\sign(g_{1:t})=\sign(g_{1:t-1})$, we have:
\begin{align*}
    |w_t-w_{t+1}|&=\left|\sign(g_{1:t-1})\left[\exp\left(\eta_{t-1}|g_{1:t-1}|\right)-1\right]-\left[\sign(g_{1:t})\exp\left(\eta_{t}|g_{1:t}|\right)-1\right]\right|\\
    &=\left|\exp\left(\eta_{t-1}|g_{1:t-1}|\right)-\exp\left(\eta_t|g_{1:t}|\right)\right|\\
    &=(|w_{t+1}|+1)\left|\exp\left(\eta_{t-1}|g_{1:t-1}|-\eta_t|g_{1:t}|\right)-1\right|
\end{align*}
where the last line uses the definition of $w_{t+1}$ to observe that $|w_{t+1}|+1=\exp(\eta_t|g_{1:t}|)$. Now we consider two cases: either $\eta_{t-1}|g_{1:t-1}|<\eta_t|g_{1:t}|$ or not. 

\noindent{\bf Case 1: $\eta_{t-1}|g_{1:t-1}|<\eta_t|g_{1:t}|$:} 

By convexity of $\exp$, we have
\begin{align*}
    |w_t-w_{t+1}|&\le (|w_{t+1}|+1)\left|\exp\left(\eta_{t-1}|g_{1:t-1}|-\eta_t|g_{1:t}|\right)-1\right|\\
    &\le (|w_{t+1}|+1)\left|\eta_{t-1}|g_{1:t-1}|-\eta_t|g_{1:t}|\right|\\
    &\le (|w_{t+1}|+1)\left(1+\frac{pL}{2}\right)\eta_t|g_t|
\end{align*}
so that the lemma holds.

\noindent{\bf Case 2: $\eta_{t-1}|g_{1:t-1}|\ge \eta_t|g_{1:t}|$:} 

Again by convexity of $\exp$ we have
\begin{align*}
    |w_t-w_{t+1}|&\le (|w_{t+1}|+1)\left|\exp\left(\eta_{t-1}|g_{1:t-1}|-\eta_t|g_{1:t}|\right)-1\right|\\
    &\le (|w_{t+1}|+1)\left|\eta_{t-1}|g_{1:t-1}|-\eta_t|g_{1:t}|\right|\exp\left(\eta_{t-1}|g_{1:t-1}|-\eta_t|g_{1:t}|\right)\\
    &\le (|w_{t+1}|+1)\left(1+\frac{pL}{2}\right)\exp\left[\eta_t|g_t|\left(1+\frac{pL}{2}\right)\right]\eta_t|g_t|
\end{align*}

so that the lemma still holds.
\end{proof}

The next lemma is the main workhorse of our regret bounds:
\begin{lemma}\label{thm:cancellingregret}
Suppose $\|g_t\|\le L$ and either of the following holds:
\begin{enumerate}
\item $p\le \frac{2}{L}$, $k=\sqrt{2}$, and $\|w_t\|\ge 15$. 
\item $k=\sqrt{2}$, $pL\ge 1$, and $\|w_t\|\ge 4\exp(p^2L^2)$.
\end{enumerate}
Then
\begin{equation}\label{eqn:balance_regret}
\left(\frac{1}{\eta_{t-1}}-\frac{1}{\eta_t}\right)\psi(w_{t+1})+g_t(w_t-w_{t+1})\le 0
\end{equation}
Further, inequality (\ref{eqn:balance_regret}) holds for any $k$ and sufficiently large $L$ if
$\|w_t\|\ge \exp((pL)^2)$.
\end{lemma}
\begin{proof}
By Theorem \ref{thm:onedimensionreduction} it suffices to consider the case $\sign(g_t)=\pm \sign(g_{1:t-1})$, so that we may adopt our identification with $\R$ and use of $|\cdot|$ throughout this proof.

For $p\le \frac{2}{L}$, $k=\sqrt{2}$ we have $15>\exp(\frac{\sqrt{pL}}{k\sqrt{2}})-1$ and for sufficiently large $L$, $\exp((pL)^2)> \exp(\frac{\sqrt{pL}}{k\sqrt{2}})-1$. Therefore in all cases $|w_t|\ge \exp(\frac{\sqrt{pL}}{k\sqrt{2}})-1$ so that by Corollary \ref{thm:wtosignbound} and Lemma \ref{thm:stability} we have
\begin{align}\label{eqn:positivepart}
    g_t\cdot(w_t-w_{t+1})\le \eta_tg_t^2(|w_{t+1}|+1)\left(1+\frac{pL}{2}\right)\exp\left[\eta_tg_t\left(1+\frac{pL}{2}\right)\right]
\end{align}

First, we prove that (\ref{eqn:balance_regret}) is guaranteed if the following holds:
\begin{align}\label{eqn:wcondition}
|w_{t+1}|+1\ge\exp\left[\frac{1+\frac{pL}{2}}{k^2}\exp\left(\eta_tg_t\left(1+\frac{pL}{2}\right)\right)+1\right]
\end{align}
The previous line (\ref{eqn:wcondition}) is equivalent to:
\begin{align}\label{eqn:logwcondition}
    k^2(\log(|w_{t+1}|+1)-1)&\ge \left(1+\frac{pL}{2}\right)\exp\left(\eta_tg_t\left(1+\frac{pL}{2}\right)\right)
\end{align}
Notice that $\psi(w_{t+1})=(|w_{t+1}|+1)(\log(|w_{t+1}|+1)-1)+1\ge (|w_{t+1}|+1)(\log(|w_{t+1}|+1)-1)$. Then multiplying (\ref{eqn:logwcondition}) by $\eta_t|g_t|$ we have
\begin{align}\label{eqn:positivepart_bound}
    (|w_{t+1}|+1)\left(1+\frac{pL}{2}\right)\exp\left[\eta_t|g_t|\left(1+\frac{pL}{2}\right)\right]\eta_t|g_t|&\le k^2\eta_t|g_t|\psi(w_{t+1})
\end{align}

Combining (\ref{eqn:positivepart}) and (\ref{eqn:positivepart_bound}), we see that (\ref{eqn:wcondition}) implies
\begin{align*}
    g_t\cdot(w_t-w_{t+1})&\le k^2\eta_tg_t^2\psi(w_{t+1})
\end{align*}

Now we bound $\left(\frac{1}{\eta_{t-1}}-\frac{1}{\eta_t}\right)\psi(w_{t+1})$:
\begin{align*}
    \frac{1}{\eta_{t-1}}-\frac{1}{\eta_t}& = k\sqrt{2}\left(\sqrt{M_{t-1}+\|g\|^2_{1:t-1}}-\sqrt{M_t+\|g\|^2_{1:t-1}+g_t^2}\right)\\
    &\le k\sqrt{2}\left(\left[\sqrt{M_t+\|g\|^2_{1:t-1}+g_t^2}-\frac{g_t^2+M_t-M_{t-1}}{2\sqrt{M_t+\|g\|^2_{1:t-1}+g_t^2}}\right]-\sqrt{M_t+\|g\|^2_{1:t-1}+g_t^2}\right)\\
    &\le -k\sqrt{2}\frac{g_t^2}{2\sqrt{M_t+\|g\|^2_{1:t}}}\\
    &= -k^2\eta_tg_t^2
\end{align*}

Thus when (\ref{eqn:wcondition}) holds we have
\begin{align*}
    \left(\frac{1}{\eta_{t-1}}-\frac{1}{\eta_t}\right)\psi(w_{t+1})+g_t(w_t-w_{t+1})&\le-k^2\eta_tg_t^2\psi(w_{t+1})+k^2\eta_t^2g_t^2\psi(w_{t+1})\le 0
\end{align*}

Therefore our objective is to show that our conditions on $w_t$ imply the condition (\ref{eqn:wcondition}) on $w_{t+1}$.

First, we bound $\eta_tg_t$ in terms of $|w_{t}|$. Notice that 
\begin{align*}
    |w_t|+1 &= \exp\left(\frac{|g_{1:t-1}|}{k\sqrt{2}\sqrt{M_{t-1}+\|g\|^2_{1:t-1}}}\right)\\
    &\le \exp\left(\frac{\sqrt{p}\sqrt{|g_{1:t-1}}}{k\sqrt{2}}\right)\\
    \frac{2k^2\log^2(|w_t|+1)}{p}&\le |g_{1:t-1}|
\end{align*}
Using this we have:
\begin{align*}
\eta_tg_t&=\frac{g_t}{k\sqrt{2}\sqrt{M_t+\|g\|^2_{1:t}}}\\
&\le \frac{g_t}{k\sqrt{2}\sqrt{M_{t-1}+\|g\|^2_{1:t-1}+g^2_t}}\\
&\le \frac{g_t\sqrt{p}}{k\sqrt{2}\sqrt{|g_{1:t-1}|+pg^2_t}}\\
&\le \frac{L \sqrt{p}}{k\sqrt{2}\sqrt{\frac{2k^2}{p}\log^2(|w_{t}|+1)+pL^2}}\\
\end{align*}

so that we can conclude:
\begin{align}\label{eqn:etagbound}
\eta_tg_t&\le \frac{Lp}{k\sqrt{2}\sqrt{2k^2\log^2(|w_t|+1)+p^2L^2}}
\end{align}

Further, by Lemma \ref{thm:diffbound} we have 

\begin{align*}
    \frac{|w_{t}|+1}{|w_{t+1}|+1}&=\exp(\eta_{t-1}|g_{1:t-1}|-\eta_t|g_{1:t}|)\\
    &\le \exp\left[\eta_tg_t\left(1+\frac{pL}{2}\right)\right]
\end{align*}

Therefore we have
\begin{align}\label{eqn:wt-to-wtplus1}
    |w_{t+1}|+1\ge (|w_t|+1)\exp\left[-\eta_tg_t\left(1+\frac{pL}{2}\right)\right]
\end{align}
From (\ref{eqn:wt-to-wtplus1}), we see that (\ref{eqn:wcondition}) is guaranteed if we have
\begin{align}\label{eqn:fullwcondition}
|w_t|+1 \ge \exp\left[\eta_tg_t\left(1+\frac{pL}{2}\right)\right]\exp\left[\frac{1+\frac{pL}{2}}{k^2}\exp\left(\eta_tg_t\left(1+\frac{pL}{2}\right)\right)+1\right]
\end{align}

If we use our expression (\ref{eqn:etagbound}) in (\ref{eqn:fullwcondition}), and assume $|w_t|\ge \exp(L^2)$, we see that there exists some constant $C$ depending on $p$ and $k$ such that the RHS of (\ref{eqn:fullwcondition}) is $O(\exp(L))$ and so (\ref{eqn:fullwcondition}) holds for sufficiently large $L$. 

For $p=2/L$, $k=\sqrt{2}$, and $w_t\ge 15$ we can verify (\ref{eqn:fullwcondition}) numerically by plugging in the bound (\ref{eqn:etagbound}). 

For the case $k=\sqrt{2}$, $|w_t|\ge 4\exp(p^2L^2)$, we notice that by using (\ref{eqn:etagbound}), we can write (\ref{eqn:fullwcondition}) entirely in terms of $pL$. Graphing both sides numerically as functions of $pL$ then allows us to verify the condition.

\end{proof}

We have one final lemma we need before we can start stating some real regret bounds. This lemma can be viewed as observing that $\psi(w)$ is roughly $\frac{1}{D}$ strongly-convex for $|w|$ not much bigger than $D$. 
\begin{lemma}\label{thm:smallw}
Suppose $p\le2/L$, $k=\sqrt{2}$, $\|w_t\|\le D$ and $\|g_t\|\le L$ Then $g_t(w_t-w_{t+1})\le 6(\max(D+1,\exp(1/2)))g_t^2\eta_t$.
\end{lemma}
\begin{proof}
By Theorem \ref{thm:onedimensionreduction} it suffices to consider $\sign(g_t)=\pm \sign(g_{1:t-1})$.

We show that $|w_t-w_{t+1}|\le 6(\max(D+1,\exp(1/2)))|g_t|\eta_t$ so that the result follows by multiplying by $|g_t|$.

From Lemma \ref{thm:diffbound}, we have $|\eta_{t-1}|g_{1:t-1}|-\eta_t|g_{1:t}||\le \eta_t|g_t|\left(1+\frac{pL}{2}\right)\le 2\eta_t|g_t|$. Further, note that $\eta_t|g_t|\le \frac{1}{k\sqrt{2}}=\frac{1}{2}$. We consider two cases, either $\sign(g_{1:t})=\sign(g_{1:t-1})$ or not.

\noindent{\bf Case 1: $\sign(g_{1:t})=\sign(g_{1:t-1})$:}
\begin{align*}
    |w_t-w_{t+1}|&=|\exp(\eta_{t-1}|g_{1:t-1}|)-\exp(\eta_t|g_{1:t}|)|\\
    &=(|w_t|+1)|\exp(\eta_t|g_{1:t}|-\eta_{t-1}|g_{1:t-1}|)-1|\\
    &\le 2(D+1)\eta_t|g_t|\exp(2\eta_t|g_t|)\\
    &\le 2(D+1)\eta_t|g_t|\exp\left(\frac{2}{k\sqrt{2}}\right)\\
    &\le 6(D+1)\eta_t|g_t|
\end{align*}

\noindent{\bf Case 2: $\sign(g_{1:t})\ne\sign(g_{1:t-1})$:} In this case, we must have $|g_{1:t}|\le |g_t|$.
Let $X = \max(\eta_t|g_{1:t}|,\eta_{t-1}|g_{1:t-1}|)$. Then by triangle inequality we have
\begin{align*}
|w_t-w_{t+1}|&\le 2\max(|w_t|,|w_{t+1}|)\\
&\le 2(\exp(X)-1)\\
&\le2X\exp(X)\\
&\le2(\max(|w_t|,|w_{t+1}|)+1)X
\end{align*}

Since $|\eta_{t-1}|g_{1:t-1}|-\eta_t|g_{1:t}||\le 2\eta_tg_t$, we have $X\le 2\eta_tg_t+\eta_t|g_{1:t}|\le 3\eta_t|g_t|$ so that we have
\begin{align*}
|w_t-w_{t+1}|&\le 6(\max(|w_t|,|w_{t+1}|)+1)\eta_t|g_t|
\end{align*}
Finally, we have $|w_{t+1}|+1=\exp(\eta_t|g_{1:t}|)\le \exp(\eta_t|g_t|)\le \exp(1/2)$, so that
\begin{align*}
    |w_t-w_{t+1}|&\le 6\eta_t|g_t|(\max(|w_t|,|w_{t+1}|)+1)\\
    &\le 6\max(D+1,\exp(1/2))\eta_t|g_t|
\end{align*}

\end{proof}

Now we are finally in a position to prove Lemma \ref{thm:firstregret}, which we re-state below:

\firstregret*

\begin{proof}[Proof of Lemma \ref{thm:firstregret}]

We combine Lemma \ref{thm:cancellingregret} with Lemma \ref{thm:smallw}: if $|w_t|\ge 15$ we have for all $t<T$:
\[
\left(\frac{1}{\eta_{t-1}}-\frac{1}{\eta_t}\right)\psi(w_{t+1})+g_t\cdot(w_t-w_{t+1}) <0
\]
and if $|w_t|\le 15$ we have
\begin{align*}
    \left(\frac{1}{\eta_{t-1}}-\frac{1}{\eta_t}\right)\psi(w_{t+1})+g_t\cdot(w_t-w_{t+1})&\le g_t\cdot(w_t-w_{t+1})\\
    &\le 6\times(15+1)\eta_tg_t^2\\
    &=96\eta_tg_t^2
\end{align*}
Therefore for all $t<T$ we have $ \left(\frac{1}{\eta_{t-1}}-\frac{1}{\eta_t}\right)\psi(w_{t+1})+g_t\cdot(w_t-w_{t+1})\le 96\eta_tg_t^2$.

\begin{align*}
    R_T(u)&\le \psi(u)/\eta_T+\sum_{t=1}^{T}\left(\frac{1}{\eta_{t-1}}-\frac{1}{\eta_t}\right)\psi(w_{t+1})+g_t\cdot(w_t-w_{t+1})\\
    &\le \psi(u)/\eta_T+96\sum_{t=1}^{T}\eta_tg_t^2+\left(\frac{1}{\eta_{T-1}}-\frac{1}{\eta_{T}}\right)\psi(w_{T+1})+g_T\cdot(w_T-w_{T+1})\\
\end{align*}

We have
\[
\left(\frac{1}{\eta_{T-1}}-\frac{1}{\eta_{T}}\right)\psi(w_{T+1})<0
\]
so that
\[
\left(\frac{1}{\eta_{T-1}}-\frac{1}{\eta_{T}}\right)\psi(w_{T+1})+g_T\cdot(w_T-w_{T+1})\le 2\Lm W_{\max}
\]

Further, again using Lemma \ref{thm:cancellingregret} we have
\begin{align*}
    \left(\frac{1}{\eta_{T-1}}-\frac{1}{\eta_{T}}\right)\psi(w_{T+1})+g_T\cdot(w_T-w_{T+1})<0
\end{align*}
for $|w_T|\ge 4\exp(p^2\Lm^2)$ since $k=\sqrt{2}$. 

Finally, notice that by definition of $\eta_t$ and $L$, we must have $|\eta_t g_{1:t}|\le \frac{\sqrt{p|g_{1:t}|}}{k\sqrt{2}}\le \sqrt{T/2}$, so that $\|w_t\|\le \exp\left(\eta_t|g_{1:t}|\right)\le \exp\left(\sqrt{T/2}\right)$. Thus we have
\[
\left(\frac{1}{\eta_{T-1}}-\frac{1}{\eta_{T}}\right)\psi(w_{T+1})+g_T\cdot(w_T-w_{T+1})\le 2\Lm \min(W_{\max},4\exp(4\Lm^2/L^2),\exp(\sqrt{2T}))
\]

Now we make the following classic argument:
\begin{align*}
    \sqrt{M_t+\|g\|^2_{1:t}}-\sqrt{M_{t-1}+\|g\|^2_{1:t-1}}&\ge \frac{g_t^2+M_t-M_{t-1}}{2\sqrt{M_t+\|g\|^2_{1:t}}}\\
    &\ge\frac{g_t^2}{2\sqrt{M_t+\|g\|^2_{1:t}}}\\
    &=\eta_tg_t^2
\end{align*}
so that we can bound:
\begin{align*}
    R_T(u)&\le  \psi(u)/\eta_T+96\sum_{t=1}^T\eta_tg_t^2+\left(\frac{1}{\eta_{T-1}}-\frac{1}{\eta_{T}}\right)\psi(w_{T+1})+g_T\cdot(w_T-w_{T+1})\\
    &\le \psi(u)/\eta_T+96\sqrt{M_T+\|g\|^2_{1:T}}+2\Lm \min(W_{\max},4\exp(4\Lm^2/L^2),\exp(\sqrt{2T}))
\end{align*}

To show the remaining two lines of the theorem, we prove by induction that $M_t+\|g\|^2_{1:t}\le L\sum_{t'=1}^t |g_{t'}|$ for all $t<T$. The statement is clearly true for $t=1$. Suppose it holds for some $t$. Then notice that $|g_{1:t+1}|\le |g_{t+1}|+|g_{1:t}|$. So we have
\begin{align*}
   M_{t+1}+\|g\|^2_{1:t+1}&=\max\left(M_{t}+\|g\|^2_{1:t+1}, \frac{|g_{1:t+1}|}{p}\right)\\
   &\le \max\left(M_{t}+\|g\|^2_{1:t}+L|g_{t+1}|,L|g_{1:t+1}|\right)\\
   &\le L\sum_{t'=1}^{t+1} |g_{t'}|
\end{align*}

Finally, we observe that $M_T = \max\left(M_{T-1}+\|g\|^2_{1:T-1}+g^2_T, \frac{|g_{1:T}|}{p}\right)\le \Lm^2+L\sum_{t=1}^{T-1} |g_{t'}|$ and the last two lines of the theorem follow immediately.

\end{proof}

\section{Additional Experimental Details}

\subsection{Hyperparameter Optimization}
For the linear classification tasks, we optimized hyperparameters in a two-step process. First, we tested every power of $10$ from $10^{-5}$ to $10^2$. Second, if $\lambda$ was the best hyperparameter setting in step 1, we additionally tested $\beta\lambda$ for $\beta\in\{0.2,0.4,0.8,2.0,4.0,6.0,8.0\}$

For the neural network models, we optimized \adam\ and \adagrad's learning rates by testing every power of $10$ from $10^{-5}$ to $10^0$. For stochastic gradient descent, we used an exponentially decaying learning rate schedule specified in Tensorflow's (\url{https://www.tensorflow.org/}) MNIST and CIFAR-10 example code.

\subsection{Coordinate-wise updates}

We proved all our results in arbitrarily many dimensions, leading to a dimension-independent regret bound. However, it is also possible to achieve dimension-dependent bounds by running an independent version of our algorithm on each coordinate. Formally, for OLO we have
\begin{align*}
    R_T(u) = \sum_{t=1}^T g_t(w_t-u)=\sum_{i=1}^d\sum_{t=1}^Tg_{t,i}(w_{t,i}-u_i)=\sum_{i=1}^d R^1_{T}(u_i)
\end{align*}
where $R^1_T$ is the regret of a 1-dimensional instance of the algorithm. This reduction can yield substantially better regret bounds when the gradients $g_t$ are known to be sparse (but can be much worse when they are not). We use this coordinate-wise update strategy for our linear classification experiments for \algname. We also considered coordinate-wise updates and non-coordinate wise updates for the other algorithms, taking the best-performing of the two.

For all algorithms in the linear classification experiments, we found that the difference between coordinate-wise and non-coordinate wise updates was not very striking. However, for the neural network experiments we found \algname\ performed extremely poorly when using coordinate-wise updates, and performed extremely well with non-coordinate wise updates. We hypothesize that this is due to a combination of non-convexity of the model and frequent resets at different times for each coordinate.

\subsection{Re-centering \algname}

For the non-convex neural network tasks we used a variant of \algname\ in which we re-center our FTRL algorithm at the beginning of each epoch. Formally, the pseudo-code is provided below:

\begin{algorithm}
   \caption{Re-centered \algname}
   \label{alg:recenteredalgname}
\begin{algorithmic}
   \STATE {\bfseries Initialize:} $k\gets\sqrt{2}$, $M_0\gets 0$, $w_1\gets 0$, $t_\star\gets 1$ , $w_\star\gets 0$
   \FOR{$t=1$ {\bfseries to} $T$}
   \STATE Play $w_t$, receive subgradient $g_t\in \partial \ell_t(w_t)$.
   \IF{$t=1$}
   \STATE $L_{1}\gets \|g_{1}\|$
    \STATE $p\gets1/L_1$  
   \ENDIF
   \STATE $M_t \gets \max(M_{t-1},\|g_{t_\star:t}\|/p-\|g\|^2_{t_\star:t})$.
   \STATE $\eta_t\gets \frac{1}{k\sqrt{2(M_t+\|g\|^2_{t_\star:t})}}$
   \STATE $w_{t+1} \gets w_\star+\argmin_w\left[\frac{\psi(w)}{\eta_t}+g_{t_\star:t}w\right]=w_\star-\frac{g_{t_\star:t}}{\|g_{t_\star:t}\|}\left[\exp(\eta_t\|g_{t_\star:t}\|)-1\right]$
   \IF{$\|g_{t}\|>2L_{t}$}
   \STATE $L_{t+1}\gets \|g_{t}\|$
   \STATE $p\gets1/L_{t+1}$  
   \STATE $t_\star\gets t+1$
   \STATE $M_t\gets 0$
   \STATE $w_{t+1}\gets0$
   \STATE $w_\star \gets w_{t-1}$

   \ELSE
   \STATE $L_{t+1}\gets L_{t}$
   \ENDIF
   
   \ENDFOR
\end{algorithmic}
\end{algorithm}

So long as $\|w_\star-u\|\le \|u\|$, this algorithm maintains the same regret bound as the non-re-centered version of \algname. While it is intuitively reasonable to expect this to occur in a stochastic setting, an adversary can easily subvert this algorithm.

\subsection{Aggregating Studies}
It is difficult to interpret the results of a study such as our linear classification experiments (see Section \ref{sec:experiments}) in which no particular algorithm is always the ``winner'' for every dataset. In particular, consider the case of an analyst who wishes to run one of these algorithms on some new dataset, and doesn't have the either the resources or inclination to implement and tune each algorithm. Which should she choose? We suggest the following heuristic: pick the algorithm with the lowest loss averaged \emph{across datasets}.

This heuristic is problematic because datasets in which all algorithms do very poorly will dominate the cross-dataset average. In order address this issue and compare losses across datasets properly, we compute a \emph{normalized loss} for each algorithm and dataset. The normalized loss for an algorithm on a dataset is given by taking the loss experienced by the algorithm on its best hyperparameter setting on that dataset divided by the lowest loss observed by any algorithm and hyperparameter setting on that dataset. Thus a normalized loss of 1 on a dataset indicates that an algorithm outperformed all other algorithms on the dataset (at least for its best hyperparameter setting). We then average the normalized loss for each algorithm across datasets to obtain the scores for each algorithm (see Table \ref{tbl:normalizedloss}).

\begin{table}
\centering
\begin{tabular}{c|c|c|c|c|c}
\adagrad&\algname&\adadelta&\scaleinvariant&\adam&\pistol\\
\hline
1.14&1.19&1.21&1.28&1.51&1.53
\end{tabular}
\vskip2pt
\caption{Average normalized loss, using best hyperparameter setting for each algorithm.}
\label{tbl:normalizedloss}
\end{table}

These data indicate that while \adagrad\ has a slight edge after tuning, \algname\ and \adadelta\ do nearly equivalently well (4\% and 6\% worse performance, respectively). Therefore we suggest that if our intrepid analyst is willing to perform some hyperparameter tuning, then \adagrad\ may be slightly better, but her choice doesn't matter too much. On the other hand, using \algname\ will allow her to skip any tuning step without compromising performance.

\end{document}